\documentclass[lettersize,onecolumn]{IEEEtran}
\usepackage{array}
\usepackage{textcomp}
\usepackage{stfloats}
\usepackage{url}
\usepackage{verbatim}
\usepackage{graphicx}
\usepackage{amsmath, amssymb}
\usepackage{amsthm}

\theoremstyle{plain}
\newtheorem{theorem}{Theorem}
\newtheorem{corollary}{Corollary}

\theoremstyle{definition}
\newtheorem{definition}{Definition}
\theoremstyle{remark}
\newtheorem{remark}{Remark}

\usepackage[mathscr]{eucal}
\DeclareMathAlphabet{\mathrsfso}{U}{rsfso}{m}{n}

\usepackage[backend=bibtex,natbib=true]{biblatex}
\usepackage{algorithm}
\usepackage{algpseudocode}
\algdef{SE}[DOWHILE]{Do}{doWhile}{\algorithmicdo}[1]{\algorithmicwhile\ #1}%
\algnewcommand{\algorithmicgoto}{\textbf{go to}}%
\algnewcommand{\Goto}[1]{\algorithmicgoto~\ref*{#1}}%

\usepackage{centernot}

\usepackage{subcaption}

\usepackage{enumitem}

\usepackage{multicol}

\usepackage{xr}
\makeatletter
\newcommand*{\addFileDependency}[1]{
  \typeout{(#1)}
  \@addtofilelist{#1}
  \IfFileExists{#1}{}{\typeout{No file #1.}}
}
\makeatother

\usepackage[linguistics]{forest}

\usepackage{bookmark}

\newcommand{\papertitle}{
    Identifying Information from Observations \\with Uncertainty and Novelty
}

\usepackage{longtable}

\usepackage{color, colortbl}

\newcommand{\vvvec}[2][3mu]{\vec{#2\mkern-#1}\mkern#1}
\newcommand{\vvec}[1]{#1}

\newcommand{\entropyRate}{\vvvec{H}}

\newcommand{\Complex}{\mathbb{C}} 
\newcommand{\C}{\Complex} 

\newcommand{\Real}{\mathbb{R}} 
\newcommand{\R}{\Real} 

\newcommand{\Computable}{\mathbb{K}} 
\newcommand{\K}{\Computable} 

\newcommand{\Rational}{\mathbb{Q}} 
\newcommand{\Q}{\Rational} 

\newcommand{\Integer}{\mathbb{Z}} 
\newcommand{\Z}{\Integer} 

\newcommand{\Natural}{\mathbb{N}} 
\newcommand{\N}{\Natural} 

\newcommand{\B}{\mathbb{B}}

\newcommand{\PS}{\mathbb{P}} 

\newcommand{\Sample}{\Omega} 
\newcommand{\Event}{\mathit{\Sigma}} 
\newcommand{\Prob}{P} 

\newcommand{\MeasurableSigma}{\mathit{\Sigma}} 

\newcommand{\Dissim}{\mathcal{D}}

\newcommand{\KL}{D_{\!K\!L} }

\newcommand{\Time}{\mathrsfso{T}}

\newcommand{\MeasurableState}{\mathrsfso{S}}
\newcommand{\MeasurableSigmaState}{\MeasurableSigma_\MeasurableState}

\newcommand{\Input}{X}
\newcommand{\IdealInput}{\Input}
\newcommand{\DataInput}{\vec{\sampleInput}}
\newcommand{\SampleInput}{\Sample_{\Input}}
\newcommand{\EventInput}{\Event_{\Input}}
\newcommand{\ProbInput}{\Prob_{\Input}}
\newcommand{\TripleInput}{(\SampleInput, \EventInput, \ProbInput)}

\newcommand{\MeasurableSigmaInput}{\MeasurableSigma_\MeasurableInput}

\newcommand{\MeasurableInput}{\mathrsfso{\Input}}

\newcommand{\sampleInput}{x}

\newcommand{\Output}{Y}

\newcommand{\MeasurableOutput}{\mathrsfso{\Output}}

\begin{document}

\title{\papertitle}

\author{Derek S. Prijatelj
    \href{https://orcid.org/0000-0002-0529-9190}{\includegraphics[height=.8em]{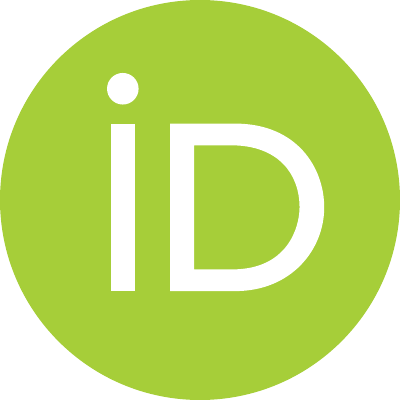}},
    Timothy J. Ireland 
    \href{https://orcid.org/0000-0003-0092-1777}{\includegraphics[height=.8em]{graphics/orcid-icon.pdf}},
    Walter J. Scheirer
    \href{https://orcid.org/0000-0001-9649-8074}{\includegraphics[height=0.8em]{graphics/orcid-icon.pdf}},
    ~\IEEEmembership{Senior Member,~IEEE}
\thanks{Manuscript received April 09, 2025;}
\thanks{Derek S. Prijatelj was with the Department of Computer Science and Engineering, University of Notre Dame, Notre Dame, IN 46556, USA (email: dprijate@nd.edu, website: \url{https://prijatelj.github.io/}).}
\thanks{Timothy J. Ireland contributed as an independent researcher (email: 4dimensionalcube@gmail.com, website: \url{https://4dimensionalcube.info/}).}
\thanks{Walter J. Scheirer is with the Department of Computer Science and Engineering, University of Notre Dame, Notre Dame, IN 46556, USA (email: walter.scheirer@nd.edu, website: \url{https://wjscheirer.com/}).}
\thanks{This research was sponsored in part by the National Science Foundation (NSF) grant CAREER-1942151 and by the Defense Advanced Research Projects Agency (DARPA) and the Army Research Office (ARO) under multiple contracts/agreements including
HR001120C0055, W911NF-20-2-0005, W911NF-20-2-0004, HQ0034-19-D-0001, W911NF2020009.
The views contained in this document are those of the authors and should not be interpreted as representing the official policies, either expressed or implied, of the DARPA or ARO, or the U.S. Government.}
}

\markboth{IEEE Transactions on Information Theory,~Vol.~?, No.~?, April~2026}%
{Shell \MakeLowercase{\textit{et al.}}: A Sample Article Using IEEEtran.cls for IEEE Journals}

\IEEEpubid{0000--0000~\copyright~2026 IEEE}

\maketitle

\begin{abstract}
    A machine that learns a task from observations must encounter and process uncertainty and novelty,
especially when it is to maintain performance when observing new information and to select the hypothesis that best fits the current observations.
In this context, some key questions arise:
    what and how much information did the observations provide,
    how much information is required to identify the data-generating process,
    how many observations remain to get that information, and
    how does a predictor determine that it has observed novel information?
We formalize \textit{identifying information} to answer these questions and synthesize prior works.
Identifying information are bits that verify or falsify a hypothesis as the data-generating process.
In this formalization, we prove the information theoretic characteristics of the computation of hypothesis identification and the resulting sample complexity.
We define hypothesis identification and sample complexity via the computation of an indicator function over a set of hypotheses, bridging algorithmic and probabilistic information.
We detail the sample complexity and its properties for data-generating processes ranging from deterministic processes to ergodic stationary stochastic processes,
which connect the notion of identifying information in finite steps with asymptotic statistics and PAC-learning.
The indicator function's computation naturally formalizes novel information and its identification from observations with respect to a hypothesis set, which detects a misspecified hypothesis set.
We also proved that a computable PAC-Bayes learners' sample complexity distribution is determined by its moments in terms of the prior probability distribution over a fixed finite hypothesis set, and thus
an approximation of the sample complexity distribution is always computable within the desired precision that resources allow.
\end{abstract}



\section{Introduction}

\IEEEPARstart{I}{dentification}
is colloquially defined as
``to determine an
    object or concept to which some observations belong.''
For example, one could identify a person from their portrait or an artist from their distinctive creative style.
The sample complexity of an identification problem is the number of observations required before an identity is determined. 
Often, we do not make an identification with absolute certainty.
In these cases, the sample complexity corresponds to the number of observations required to reach an \textit{a priori} level of certainty regarding the identity.
Most humans have a powerful memory for faces; identification of a known person from a portrait generally requires only a single image, so we would say that the sample complexity of this problem is one.
However, what happens when we can only observe part of the portrait, or when the image is of poor quality?
We concern ourselves with this sort of inference problem, where we want to characterize the amount of evidence required to overcome general identification problems.
As we will see, an extremely broad class of mathematical and statistical problems can be characterized as identification problems and thus addressed with the formalism we present.
This framework will ensure that we have appropriate language and clear assumptions such that the information present in observed phenomena relative to an identification problem can be precisely quantified and modeled.

Thanks to the ubiquity of identification problems, the theory of identification and sample complexity lacks unity and rigor across algorithmic and probabilistic contexts.
This becomes extremely clear when we wish to quantify the information that is contained within the model's description and the observations.
Most texts call models ``identifiable" when they may be differentiated based on observations.  
In the probability literature, this definition frequently includes that their probability distributions differ, as done by~van~der~Vaart~\cite[Eq.~5.34]{vaartAsymptoticStatistics2000}.
While this definition serves as an appropriate starting point, the nuance in how the observations effect model identifiability is 
    not as evident as telling two observations apart and can greatly vary for different hypothesis sets.
This complexity has resulted in some recent work that attempts to directly elucidate model identifiability, such as \citet{lewbelIdentificationZooMeanings2019}.
To identify an abstract set of observations with respect to a proposed model is to determine that the observations are consistent with what would be observed were the observations drawn from that model.
The theory of computability provides a rigorous foundation for identification in deterministic contexts with full observability, as we can treat this sort of identification problem exactly as we would approach the computation of a Boolean indicator function over some domain.
We reinforce the connection between computational theory and identifiable models in statistics by working from computational identification through asymptotic identification up to probably approximately correct identification from statistical learning theory.
    
Statistical learning theory is primarily concerned with the sample complexity, which is the number of observations required to identify a model from some explicit or implicit set of hypotheses~\citep{shalev-shwartzUnderstandingMachineLearning2014}.
Sample complexity as a term seems to originate from \citet{hausslerQuantifyingInductiveBias1988} after \citet{valiantTheoryLearnable1984} introduced what \citet{angluinQueriesConceptLearning1988} and \citet{angluinLearningNoisyExamples1988} called and is now popularly known as Probably Approximately Correct (PAC) learning.
\citet{valiantTheoryLearnable1984} specified PAC with a polynomial computational complexity constraint and an initial sample size constraint to be polynomial in the hypothesis set's cardinality.
Before then, the related term ``sample size'' was used for the samples that meet the desired empirical risk minimization considered in the 1966 work of~\citet{vapnikUniformConvergenceRelative2015} and the 1971 English translation that founded VC-Theory.
The sample size was also included in following complexity analysis works, e.g., \citet{pearlConnectionComplexityCredibility1978a}.
To better clarify identifiability and sample complexity we formalize these concepts in information-theoretic terms in both algorithmic and probabilistic contexts.
This results in a more thorough understanding of the information present in the model's description and how that information is expressed via the observations.

Novelty and the unknown have concerned scholars since antiquity, as evidenced by \citeauthor{platoRepublic1998}'s ``Allegory of the Cave''~\citeyearpar[Book VII]{platoRepublic1998}, \citeauthor{platoApology1999}'s ``Apology''~\citeyearpar{platoApology1999}, and the reasoning about the unobserved or ``nonapparents'' from signs~\cite[Ch.~10]{asmisEpicurusScientificMethod1984}.
There has been a resurgence of academic interest in these topics within the artificial intelligence and machine learning literature over the past three decades, as expressed by \citet{boultUnifyingFrameworkFormal2021, boultUnifyingFrameworkFormal2024}.
Practical problems often involve unseen events with respect to the in-sample observations or prior information, and successful machine learning algorithms need to continue to perform well in the presence of such events.
Novelty is associated with many subtopics in machine learning, including novelty detection~\citep{markouNoveltyDetectionReview2003,markouNoveltyDetectionReview2003a,pimentelReviewNoveltyDetection2014,fariaNoveltyDetectionData2016}, open world learning~\citep{bendaleOpenWorldRecognition2015,pintoMeasuringPerformanceOpenWorld2022,doctorDefiningDomainComplexity2023,langleyOpenWorldLearningRadically2020}, and anomaly, outlier, or out-of-distribution detection~\citep{ruffUnifyingReviewDeep2021,chandolaAnomalyDetectionSurvey2009}.
Novelty is also related to the concepts of ``discovery'' and ``emergence'' in complex systems as studied in statistical and computational mechanics~\citep{crutchfieldCalculiEmergenceComputation1994}. 
We refer to ``novelty'' informally as ``something that is unknown to another thing at a moment in time,'' which is consistent with prior work and for which we further formalize using set theory and information theory~\citep{shannonMathematicalTheoryCommunication1948, coverElementsInformationTheory1991, mackayInformationTheoryInference2003}.
Generalization of the predictor's learned patterns inherently involves incorporation of or robustness to novel information, as evidenced by \citet{zhangRobustPatternRecognition2020}.
We emphasize the presence of novelty throughout our formalization of identifying information and determine the sample complexity to identify both known and unknown models.
Our theory formally addresses identifying novel information with respect to a model and its hypothesis set, 
such that a predictor can detect if its hypothesis set is misspecified, which is a problem that plagues Bayesian estimators and PAC-Bayes learning~\citep{kleijnMisspecificationInfinitedimensionalBayesian2006,kleijnBernsteinVonMisesTheoremMisspecification2012,walkerBayesianInferenceMisspecified2013,wangVariationalBayesModel2019,lotfiBayesianModelSelection2022}.

\subsection{Contributions}
After introducing our notation and problem setting, we present our
information theoretic formalization of the computation of \textbf{hypothesis identification} and the resulting \textbf{sample complexity} distribution
in the following structure:
\begin{enumerate}
    \item In Section~\ref{sec:what_is_novelty},
        we generalize
        the novelty spaces from \citet{boultUnifyingFrameworkFormal2021} by reducing them into a single measure space
        coupled with the observations within that space and transition functions between the spaces to allow their state to affect each other.
        A random variable or stochastic process
        naturally generalizes that measure space connecting to information theory.
        The definition of novelty is dependent upon the indicator function, preparing us to formally discuss novel information as we formalize the computation of hypothesis identification and the sample complexity.
        \item In Section \ref{sec:id_info},
        we generalize the definitions of hypothesis identifiability and sample complexity in terms of information up to stationary ergodic stochastic processes, unifying the algorithmic and probabilistic cases.
        \item  In Section \ref{sec:id_direct}, we prove our definitions hold in the direct observation case whether considering finite or infinite observation sequences or hypothesis sets. We transition to probability through the hypergeometric distribution.
        \item  In Section \ref{sec:id_comp}, we briefly prove the connection between hypothesis identification and computability.
        \item  In Section \ref{sec:finite_over_infinite}, we prove how the identifying information spreads over infinite observations in the practical case of indirect obseravtions and prove a generalization of asymptotic identification.
        \item  In Section \ref{sec:part-id_distrib_sc}, we prove our sample complexity generalization for PAC-Bayes, which lets us prove how to compute the sample complexity until falsification of the hypothesis set given the observations and thus detect a misspecified hypothesis space.
\end{enumerate}

\pagebreak
\section{Notation}
\label{bg:prob}

As this topic is inherently interdisciplinary, we recommend the reader to be familiar with relevant concepts in machine learning, probability, statistics, and computational theory along with their connections to information theory.
A review of the relevant topics is provided in Appendix~\ref{sec:app:bg}.
Note that we use $\entropyRate(X)$ to denote the entropy rate of a process, as in Table~\ref{tab:info_theory_bg}.
See Table~\ref{tab:syn_map} for an approximate synonym mapping of important terms across disciplines, and Appendix~\ref{sec:app:bg} for further details.

We model both the data generating process and the predictor as computational stochastic processes.
The hypothesis set consists of models represented as programs that, when coupled with an entropy source, implement candidate data generating processes for the observations.
The predictor is a program tasked with either identifying the best hypothesis in its hypothesis set or exhaustively falsifying each hypothesis given the observations.
Parameterized models form families of processes, as in the case of parametric statistical models.
It is important to note that not all parametric models used in statistics have identifiable parameters by design, which makes determining identifiable subsets important for learning equivalent hypotheses or their parts.
We often refer to identifying a hypothesis, which is equivalent to identifying a model, a parameter for a parametric model, or an algorithmic description.

We often concern ourselves with binary strings where a bit is within the Boolean domain $\B \triangleq \{0,1\}$.
The set of $L$ length bit strings are denoted $\B^L$ and the set of arbitrary length bit strings are denoted with a Kleene star $\B^* \triangleq \cup_{i \ge  0} \B^i$.
The space of the computable numbers $\K$ is the union of the computable irrationals with the rationals $\Q \triangleq \frac{\Z}{\Z^+}$, where $\Z$ is the set of integers,
    $\Z^+$ is the set of positive integers $\{1, 2, 3,...\}$.
$\N$ is the natural numbers starting at $0$, $\N \triangleq \{0\} \cup \Z^+$. 
The proper subset relationship chain from $\N$ to complex numbers $\C$ is
$
    \N \subset \Z \subset \Q \subset \K \subset \R \subset \C
$.

\begin{figure*}[t]
    \centering
    \subcaptionbox{
        $\PS^2$ of $(a, b)$
        \label{fig:prob_simplex:2}
    }[.2\textwidth]{
        \centering
        \includegraphics[width=.8\linewidth]{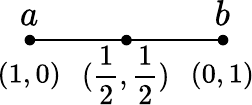}
    }
    \hfill
    \subcaptionbox{
        $\PS^3$ of $(a, b, c)$
        \label{fig:prob_simplex:3}
    }[.3\textwidth]{
        \centering
        \includegraphics[width=.9\linewidth]{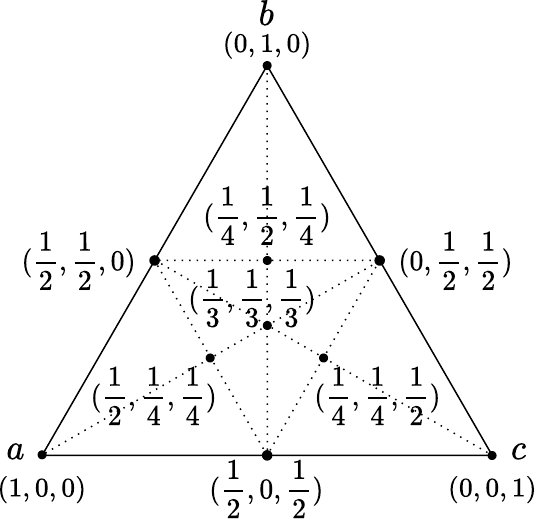}
    }
    \hfill
    \subcaptionbox{
        $\PS^4$ of $(?, a, b, c)$
        \label{fig:prob_simplex:4}
    }[.33\textwidth]{
        \centering
        \includegraphics[width=.9\linewidth]{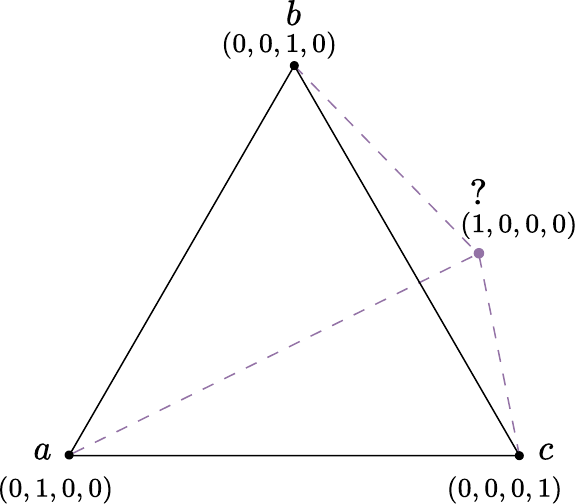}
    }
    \caption{
        \small
        Example probability vectors in their Barycentric coordinates of the probability simplex $\PS^k$ of $k$ mutually exclusive symbols.
        The traditional empirical process's $\PS^k$ of
            $k$ known symbols
            is a subspace of the empirical process's $\PS^{k+1}$ with a symbol `$?$' to represent unknown observations.
    }
    \label{fig:prob_simplex}
\end{figure*}

We rely upon the common set-theoretic definitions of probability and random variables as by \citet{doobMeasureTheory2012} and \citet{taoIntroductionMeasureTheory2011}.
We often refer to a discrete $(\MeasurableInput, \Sigma_\MeasurableInput)$-valued random variable $X$ with its associated probability triple $\TripleInput$ and its sample function $S_\Input: [0, 1] \mapsto \Omega_X$.
Random variables generalize deterministic variables, which can be viewed as random variables with an event space containing only a single element.
We use random variables to define the unknown and potentially random observed states at each time-step.
Given our focus on computable learning, we will consider discrete sequences of observations and define probability distributions over those discrete events.
The probability simplex $\PS^k$
    that contains the possible probability mass function, as a vector $p$, of some discrete random variable with support of $k$ mutually exclusive symbols, 
    is defined as
\begin{equation*}
    \mathbb{P}^k \triangleq
    \bigg\{p:
        \sum^k_{i=1} p_i = 1
        ,~ p_i \in [0, 1] 
        ,~ \forall i \in \{1,2,..., k-1, k\}
    \bigg\}
\end{equation*}
where $k$ is the number of classes
    in a classification task.
$\PS^k$ is equilateral in that all edge lengths are equal.
Figure~\ref{fig:prob_simplex} visualizes example probability vectors in their Barycentric coordinates of $\PS^k$.

\begin{table}[b]
    \centering
    \caption{\textsc{Approximate Interdisciplinary Synonym Map}}
    \fontsize{10}{10}\selectfont
    \def\arraystretch{1.2}
    \begin{tabular}{l|l|l|l}
        \textbf{Probability \& Statistics} &
        \textbf{Statistical Learning} &
        \textbf{Machine Learning} &
        \textbf{Computational Theory}
        \\
        \hline
        identifiable statistical model
        &
        identifiable hypothesis
        &
        identifiable concept
        &
        computable function
        
        \\
        \hline
        model/hypothesis selection,
        &
        hypothesis selection,
        &
        converged model,
        &
        estimated function
        \\
        converged model/estimator
        &
        learned concept
        &
        learned concept
        &
        \\
        \hline
        
        model,
        stochastic process
        &
        hypothesis
        &
        parameter coordinates
        &
        program/algorithm
        
        \\
        \hline
        parameterized model,
        set of processes
        &
        hypothesis set/space
        &
        parameter space
        &
        equivalent algorithm set
        \\
    \end{tabular}
    \label{tab:syn_map}
\end{table}

In this work, we consider stochastic processes that represent a time-ordered series of unknown, potentially random, quantities represented as an ordered sequence of random variables.
Let all processes be referenced by their latent state sequence $X$ with respect to time, or their latent state set $\Grave{X}$ with respect to its own index set.
Let $\vvec{X}^{L}_t$ indicate an $L$ length time ordered sequence possibly with repetition that \textit{inclusively} starts at the index in the subscript, where
$\vvec{X}^{1}_t = \vvec{X}_t$.
To refer to $L$ prior states, $\vvec{X}^{-L}_t$ is the ordered $L$ length sequence ending \textit{exclusively} at $t$.
$\vvec{X}^0$ is the empty set.
Boldfaced letters, $\vvec{\mathbf{X}}^t_1$, represent a set of time index-aligned processes over $\MeasurableInput^t$.
The length of bits $\ell(\vvec{x}^t_1) \neq t$ when $\MeasurableInput \neq \B$.
The general definition of a process as per \citet{bellmanAdaptiveControlProcesses1961} is as follows in our terminology and notation.
\begin{definition}
\label{def:process}
A \textbf{process} $\vvec{\Input}$
    consists of the following
    parts in order of dependency in definition,
where it is a \textbf{deterministic process} if there is no randomness in either the states or the transition function.
Otherwise, 
    the process is considered a \textbf{stochastic process}.
\begin{enumerate}[noitemsep,nolistsep]
    \item \textbf{Time} $t \in \Time : \Time \subseteq \Z$;
        There is some index-set that represents time.
        As we focus on computable processes, this is discrete time.
        
    \item \textbf{Latent state set} $\Grave{\Input} \triangleq \{\Grave{\Input}_1, \Grave{\Input}_2, ..., \Grave{\Input}_m \}$;
        A set of $m$ potentially random variables that each serve as a unique state of the (stochastic) process
        where a state may possibly repeat in the process throughout time. 
        This set may be indexed by some index-set, which is not necessarily time.
        
    \item \textbf{Observations} $\vvec{\sampleInput}^{L}_t \triangleq (\vvec{\sampleInput}_i = S(\vvec{X}_i) : i \in (t, t+1, ..., L-1, L))$;
        the observed states of a process is the time-ordered sequence of samples of the latent states.
        If the latent state is deterministic, then, in the computable setting, the observed state is the output of a Turing machine given the latent state description.
    
    \item \textbf{Transition function}
        $T_{\Input}(t, \vvec{\Input}^{-L}_t, \vvec{\sampleInput}^{-L}_t) = \Input_t$;
        The transition function is 
            a map of the current state and all necessary dependent prior states to the next latent state.
        $T_\Input$ may be a random function, meaning there could be inherent uncertainty independent of the function's parameters that affects determining the next state.
\end{enumerate}
\end{definition}

\begin{figure*}[t]
    \centering
    \includegraphics[width=.9\linewidth]{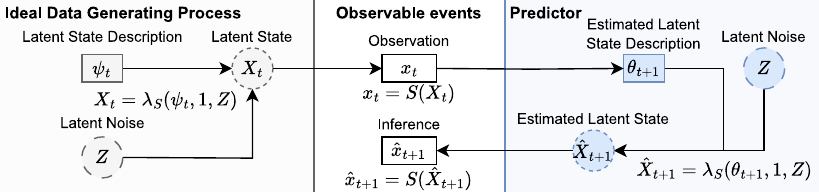}
    \caption{
        \small
        A single time-step of a predictor updating its estimated hypothesis that best describes the observations up to this point in time.
    }
    \label{fig:infer}
\end{figure*}

We treat time-step $0$ as the time of the predictor's initialization with any prior knowledge and time-step $1$ is the predictor's first observation.
    When the model is a fixed description and latent state independent of time and prior state, then the result is an
    i.i.d.
    random variable, or a repeated sequence of a constant if deterministic, as the transition function is then the identity function given the latent state.
The ideal data-generating process is denoted by $X$ and is independent of the predictor.
We use the hat notation to indicate the predictor's estimated versions of the ideal elements, such as an inferred observed state $\hat{x}$ and the estimated random variable $\hat{X}$.
The latent state descriptions of are denoted by $psi$ for the ideal and $\theta$ for the predictor, where the predictor's hypothesis set is $\Theta$.
In the direct observation case $\Theta = X$.
The descriptions are often probability distributions in $\PS^k$.
Figure~\ref{fig:infer} depicts this notation in use for a single time-step of a predictor updating its internal state.

\section{Problem Setting \& Statement}

This paper is concerned with the computation of hypothesis identification and the resulting sample complexity distribution.
The sample complexity informs how many more observations are required for the predictor to satisfy \textit{a priori} probability and correctness thresholds to identify a hypothesis, i.e., either verify that a hypothesis in the set satisfactorily describes the observations within the accepted probability and correctness measures or falsify all of the hypotheses in the set given the observations.
The sample complexity \textit{distribution} is the probability distirbution over the number ($\N$) of observations that remain until hypothesis identification occurs.
The sample complexity is determined by the amount of information the future possible observations provide to identify the hypotheses from each other.
There is also the issue of a \textit{misspecified} hypothesis set,
which is when the ideal process is not adequately described by any hypothesis in the hypothesis set
~\citep[Ch.~8.5]{ghosalFundamentalsNonparametricBayesian2017} and is known  in statistical learning as ``agnostic learning'' where the predictor's best possible minimum error is greater than zero~\cite[Def.~3.3]{shalev-shwartzUnderstandingMachineLearning2014}.
The global minimum error can be zero only if the correct hypothesis is in the set, 
which is referred to as ``realizable learning'' in statistical learning~\cite[Def.~2.1]{shalev-shwartzUnderstandingMachineLearning2014}.

The problem in the current literature across statistical learning and other mathematical modeling fields is that the sample complexity distribution is only ever indirectly estimated or only its properties are considered, such as bounds on its expected value.
Existing work rarely considers sample complexity in as general a case as stationary ergodic stocahstic processes.
We prove the sample complexity can be computed within the \textit{a priori} accepted probability and correctness measures for a predictor's learning of such processes and are able to do so by considering the computational constraints of machine learning.
With the sample complexity defined for a well-specified hypothesis set, we can then use it to inform us when the observations are atypical to the hypotheses considered, letting us detect novelty and thus know when the hypothesis set is misspecified.
Before formalizing identification and the sample complexity within terms of computation and information theory, we address the novelty in terms of measure theory and its computation using the indicator function, and then we define the common objects used throughout the paper.

\subsection{What is Known, Unknown, and Novel?}
\label{sec:what_is_novelty}

To construct a generic space in which uncertainty and novelty occur for a phenomenon,
    we use a straightforward generalization of the spaces from
    \citet{boultUnifyingFrameworkFormal2021}.
This generalizes their three spaces for the world, perception, and agent internals into one space that is necessary and sufficient by definition to measure novelty that occurs within that space over time, as long as the space and measure capture the novelty of interest.
Having this abstracted space in which novelty may occur allows for using it as a component for specifying an arbitrary combination of such spaces to discuss what is novel where and with respect to which observer or what evidence.
\begin{enumerate}
    \item
    \textbf{State Space} $(\MeasurableState, \MeasurableSigmaState)$;
        The space in which the states of the phenomenon occur is a measurable space which consists of a set of possible observations $\MeasurableState$ and a $\sigma$-algebra on that set which specifies the subsets of that set that can be measured~\citep[Ch.~1.4]{taoIntroductionMeasureTheory2011}. 
    \item
    \textbf{Dissimilarity Measure} $\Dissim_{\MeasurableState}$: $\MeasurableSigmaState \mapsto [0, +\infty)$;
        the quantified dissimilarity between two states.
        The dissimilarity measure may be selected to tell apart all types of state in the measurable space, or may be chosen to correspond to the task to ignore task-irrelevant changes in the state.
        A non-binary measure may have a threshold or similar mapping applied to acquire a binary measure of dissimilarity, which may be represented by an indicator function.
        Notably, a measure on the measurable space forms a measure space $(\MeasurableState, \MeasurableSigmaState, \Dissim_{\MeasurableState})$~\cite[Ch.~1.4]{taoIntroductionMeasureTheory2011}, \cite[Ch.~3.1]{doobMeasureTheory2012}.
    \item
    \textbf{Observations} $\vvec{s}^{t}_1$
        a sequence of observations that serves as a recorded history of previous states within $(\MeasurableState, \MeasurableSigmaState)$.
        Ideally, this is a map of time to prior states, but a dataset as a mathematical set is still experience simply without the samples ordered by time within, thus unable to say which sample came before another, but if future data occurs, then such a relative ordered comparison may be made.
        Time could be discrete resulting in the experience as a time-indexed tuple of states, or a tensor where one dimension represents time as in \citet{boultUnifyingFrameworkFormal2021}.
\end{enumerate}
These three elements allow one to determine if a state is novel relative to the observations of the phenomenon, which is directly mapped to the data of the observed phenomenon as the experience $\DataInput$ within a measure space $(\MeasurableInput, \MeasurableSigmaInput, \Dissim_{\MeasurableInput})$.
Notably, the probability space $\TripleInput$ defined before is a (probability) measure space, and there can be different measures on the same measurable-space forming different measure spaces.
When combined with a model of the observed phenomenon as the random variable $\hat{X}$, novelty of future observed events relative to the predictor may be expressed in both probability and information theory.
If the model is instead the ideal $\IdealInput$, then this is the omniscient evaluator knowledge of the phenomenon and novelty.

Each space defined by \citet{boultUnifyingFrameworkFormal2021} may be defined individually using a measure space and observations of events within that space.
Then a set of functions are specified that maps these state spaces to each other, and, in doing so, defining the dynamics of the system.
Modeling a system this way enables constructing as simple or as complex a model as desired for either the ideal system
    or the predictor's learned system.
The simplest case involves only two spaces, the ideal system (world) and a predictor (agent), where the channel or measurements used for observation are perfect,
meaning no added or irrelevant information in transfer from the observations
to the predictor, or in the predictor's chosen actions to a subset of the ideal if the predictor has agency.
This may be seen as the teacher-student pairing in teacher student learning~\cite[ch. 19.3.7]{murphyProbabilisticMachineLearning2022}.
We focus on this case where the predictor is to learn the description of exactly what it observes, without its internal state or outputs effecting what it observes.

\subsection{On the Indicator Function and Determining Novelty}
With the connection of \citet{boultUnifyingFrameworkFormal2021} to measure theory, we can discuss what it means to measure novelty.
The simplest case of determining known and unknown events is if the events are in the known set or not, which is determined by the observations and any prior knowledge incorporated into the model of that space.
The set of unknown events is then any event that is not in the predictor's current known set,
which captures the relative nature of the unknown and novelty to what is known at a moment in time.
An \textit{indicator function} $\mathbf{1}(X_t, x)$ determines if an event $x$ is known or unknown at a moment in time with respect to some known set $X_t \subseteq \MeasurableInput$, and thus an event is to be considered novel when not in the known set
\begin{equation*}
\label{eq:indicator}
    \mathbf{1}(X_t, x) = 
    \begin{cases}
        1 & \text{if~~} x \in X_t \\
        0 & \text{if~~} x \notin X_t
    \end{cases}
\end{equation*}
A novel event is then simply an event that is unknown at a moment in time.
Recalling our informal definition of novelty in the introduction, ``something is new to another thing,'' the first ``something'' is the observed event, and  ``another thing'' is a set of known events, such as the predictor's event space or an evaluator's record of observations.
The temporal aspect ``new'' is relative to the current time-step $t$ of $X_t$.
We use $\mathbf{1}(\vvec{X}, \vvec{x})$ for a single process and $\mathbf{1}(\vvec{\mathbf{X}}, \vvec{x})$ for a set of processes to indicate if the observed sequence belongs to any of the processes' sample spaces.
More often we will use this function input signature for the first two inputs of the sample complexity function $\mathbf{i}(\cdot)$ as later defined in Definition~\ref{def:sample_complexity}.

The unknown truth value of a logical statement in a moment in time has been considered at least as early as \citet{aristotleOnInterpretation}.
He considered the undetermined state of logical expressions due to not enough information being currently available to determine whether the statement was either true or false.
This may be viewed as an ill-defined problem, and it is in terms of deductive logic as the information is incomplete.
In inductive inference, it is important to consider how new information updates that which is already known and if a statement's incomplete information has become complete or complete enough to make a decision with high certainty.
When something is uncertain that means the actual value is unknown and the degree of certainty may be expressed in terms of probability or likelihood. 
Something uncertain is currently undetermined, regardless of whether the system is deterministic or nondeterministic.
Uncertainty is about a lack of complete information.
A degree of certainty is about the presence of partial information.
We use the following terminology:
\begin{itemize}
    \item \textbf{Verified:}  logical expressions are determined to be true.
    \item \textbf{Falsified:} logical expressions are determined to be false.
    \item \textbf{Undetermined:} logical expressions are yet to be determined as true or false.
\end{itemize}
If a model has yet to be falsified, then it is not to be discarded as it may still be true, as per the \textit{principle of multiple explanations}~\cite[Ch.~18]{asmisEpicurusScientificMethod1984}.
In line with this fallibilistic perspective~\cite{haackTwoFallibilistsSearch1977,buchler2014}, we use falsification akin to \citet{popperLogicScientificDiscovery2002} as to be able to falsify a known model or hypothesis is crucial to know that the evidence indicates there is an unknown model yet to be discovered that better describes the observations.
A useful distinction between undetermined and undefined expressions is that undetermined is yet to be determined, but still can be determined, while undefined is the absence of a description to express an expression's semantics.
Undefined expressions include $\frac{a}{0} : a \in \R$, or the probability of some event given the empirical process upon initialization without any prior occurrences observed.
In Section~\ref{sec:id_info}, we modify the algorithm that computes an indicator function to measure the distinct states of complex objects represented as strings within some language and extend this concept to asymptotic statistics and PAC learning.
Measuring novelty then turns from a Boolean measure to a counting measure and then finally to a general computable measure of information.

\subsection{
    Identifying Information and the Sample Complexity
}
\label{sec:id_info}

\begin{figure*}[t]
    \centering
    \includegraphics[width=.9\linewidth]{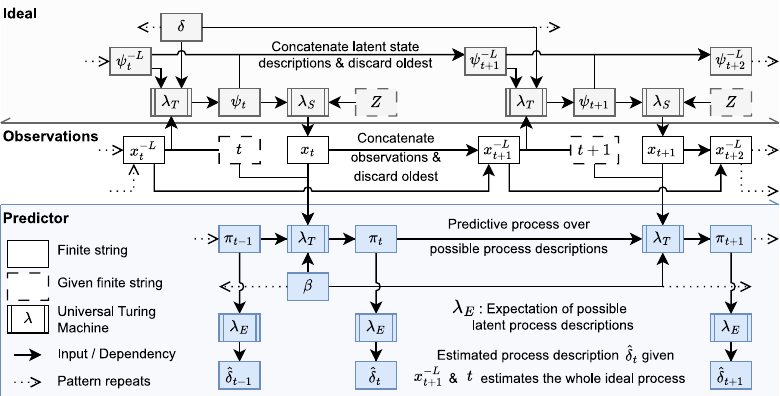}
    \caption{
        \small
        The computation of the predictor updating its internal state given observations of the unknown ideal process.
    }
    \label{fig:cbn_observe}
\end{figure*}

\label{sec:id_info_defs}
We define the concept of identifiable hypotheses in terms of relative information in order to quantify the degree to which a hypothesis may be verified or falsified within some system.
We aim to construct these definitions in such a way that there is an intuitive connection between identifying information and the computation of the indicator function.
Let $\lambda$ be a Turing machine that computes which element $\theta$ in a set of hypotheses $\mathbf{\Theta}$ most probably generated observations $\vvec{x}^{t}_1$.
Note that $t \ge 1$ is the total number of samples, and the computation of the sample function is denoted as $\lambda_S(\theta, t, \vvec{z}) = \vvec{x}^t_1$ where $\vvec{z} \sim Z$ is an entropy source that may be ignored and will be covered in depth in Section~\ref{sec:finite_over_infinite}.
We use the set $\vvec{\mathbf{X}}^t_1 = \{\vvec{X}^t_1 = \lambda_S(\theta, t, Z) : \theta \in \mathbf{\Theta} \} $ to represent a set of hypothesized observable processes with descriptions in $\mathbf{\Theta}$.
To simplify our state representation, we consider Turing machines with separate input, memory, and output tapes, which are reducible to single tape Turing machines~\cite[Ch.~6.3.3]{minskyComputationFiniteInfinite1967}.
In the following, the input tape will always contain exclusively the description of the output.
We focus on finite descriptions and finite memory tapes, but will address how relaxing these restrictions can affect identification.
Of particular interest are scenarios with infinite output but finite input, e.g., machines that output the digits of computable irrational numbers.
We also let the computation time be unbounded so we can focus on when and how a finite sample complexity can be ensured,
noting that identification with absolute certainty can only occur when the time to identify is surely finite, i.e., with surely finite sample complexity.

Figure~\ref{fig:cbn_observe} depicts the process of a predictor updating its internal state with each new observation of the ideal process.
$\beta$ is the string description of the predictor's learning algorithm, which remains constant.
$\pi$ is the predictor's description of its internal state, e.g. a representation of a posterior over the hypothesis set for a Bayesian predictor.
The computation of the expected hypothesis in the set, represented $\lambda_E$, produces a description $\hat{\delta}$ which can be used alongside our sampling function for inference. 
However, we focus primarily on how the predictor's internal state converges in the case of verifying a hypothesis or fails to converge in the case of falsifying the hypothesis set given the observations.

We start with the common point identifiable parameter of a statistical model as per van~der~\citet[Eq.~5.34]{vaartAsymptoticStatistics2000} and extend that to all the parameters (hypotheses) in a set.

\begin{definition}
\label{def:id_param}
    (van~der~\citet{vaartAsymptoticStatistics2000},~Eq.~5.34)
    A $\mathbf{\Theta}$ parameterized family or set of probability distributions defined by the model $\mathbf{P}_\mathbf{\Theta}:\mathbf{\Theta} \times \MeasurableInput \mapsto [0,1]$
    is said to have a (point) \textbf{identifiable parameter} $\theta_i \in \mathbf{\Theta}$
    when
    $\theta_i$ results in a unique probability distribution different from all other parameters' resulting probability distributions,
    \begin{equation*}
        \exists (\theta_i \in \mathbf{\Theta}), \forall (\theta_j \in \mathbf{\Theta}: \theta_i \neq \theta_j ) :
            P_{\theta_i} \neq P_{\theta_j}
    \end{equation*}
\end{definition}

\begin{definition}
\label{def:id_space}
    When
        all of the parameters in the set are identifiable parameters, the model has a \textbf{fully identifiable parameter set} $\mathbf{\Theta}$
    where
         every parameter results in a unique probability distribution.
    \begin{equation*}
        \forall (\theta_i, \theta_j \in \mathbf{\Theta}) :
            \theta_i \neq \theta_j \iff P_{\theta_i} \neq P_{\theta_j}
    \end{equation*}
\end{definition}
    
\citet{lewbelIdentificationZooMeanings2019} in Section~3.2
provides a more nuanced definition of identifiable, which we redefine in our own words here, emphasizing that what makes it identifiable is that the parameter $\theta_i$ \textit{is uniquely determined by the observations} and thus the identifiable parameters are \textbf{not} \textit{observationally equivalent} to other parameters.
\begin{definition}
    A set of hypotheses are
    \textbf{observationally equivalent} with respect to some observations
        if no hypothesis in the set can be verified or falsified via those observations, given some dissimilarity measure $D$ of the observations and hypotheses.
\end{definition}
\begin{remark}
Synonyms are ``unidentified'', ``undetermined'', ``indistinguishable'', and ``degenerate''.
In physics, two or more states are ``degenerate'' when they have the same measurement, which is often their energy in quantum mechanics~\cite[Ch.~4.1.3]{sakuraiModernQuantumMechanics2020}.
The measure provides a different enumeration of the distinct states and if two states have the same measurement, then they are observationally equivalent with respect to that measure.
In machine learning, we always want to find the observationally equivalent set of best performing models given our performance measures.
We can then find the most resource efficient models within that set, and if we can identify correct parts shared across that hypothesis subset, then we can better understand the structure of the underlying concepts.
\end{remark}

\begin{definition}
   An observationally equivalent subset of parameters, while not identifiable amongst themselves, may together be identifiable from some larger set of parameters. In this case that subset is deemed \textbf{partially identifiable} within the larger parameter set.
\end{definition}

We note that the dimensionality of observations does not appear in these definitions; thus, we can replace the distributions with stochastic processes without loss of generality. 
Identification requires only the existence of discoverable information that differentiates distinct processes from each other.
This information might be contained directly in the observable output.
Alternatively, it could be latent, e.g., describing the sampling procedure to generate the observations or indexing a model from a universe of possible models.
In either case, we call this the \textit{identifying information}.
We quantify this identifying information by measuring the number of observations that would be required by an effective procedure that determines whether a process is a member of a given set. 
In the case of latent identifying information as in Section~\ref{sec:finite_over_infinite}, we illustrate how that information can be spread over an infinite process and formalize the notion of a probabilistic procedure for identification.
We can then examine the distribution of the observations required to satisfy a probabilistic decision criterion.
We now introduce the deterministic case of identification, which we further generalize later.

\begin{definition}
    \textbf{Identification} is the verification of an observed element's membership in a set, either by verifying that it is equal to a member of the set, or by exhaustively falsifying equality with all members of the set. 
\end{definition}
\begin{definition}
\label{def:partial_id}
    \textbf{Partial identification} is the falsification of an observed element's equality with a strict subset of elements, such that the complement of that subset contains elements that have yet to be falsified or verified. This complement is the \textbf{partially identified subset}.
\end{definition}
In the context of novelty, what is unknown is determined by exhaustive falsification of what is known.
Novelty can then be described in terms of a series of symbol comparisons, where candidates in the known set are falsified one symbol at a time until novel information is encountered, addressed in Section~\ref{sec:id_nov_direct} after covering the computation of direct identification.
Probable novelty is more nuanced and is covered in Section~\ref{sec:prob_exhaust_false}.

For now, we define sample complexity in its most general form for stationary ergodic processes that includes as special cases the deterministic case, asymptotic case, and PAC case of identification, all of which we address in the following subsections.
Our generalized sample complexity definition
relies on typical sets, which are sets of observation sequences with probability that are representative samples of the block entropy of that probability distribution, i.e., the log probability of the sequences is near that block entropy.

\begin{definition}
(\citet{coverElementsInformationTheory1991}~Ch.~3.1)
The \textbf{typical set} $A^{L}_\epsilon : L \in \Z^+$ with respect to $P(X^L_1)$ is the set of observable sequences $x^L_1 \in \MeasurableInput^L$ with the property
$$2^{-L(\entropyRate(X)+\epsilon)} \le P(X^{L}_{1} = x^{L}_{1}) \le 2^{-L(\entropyRate(X)-\epsilon)}$$
\end{definition}

Thresholds $p$ and $q$ define two typical sets $\mathcal{A}_p \subseteq \mathcal{A}_q$, where equal only if $p=q$, for a known process where the most probable sequences all have probability with some deviation from the entropy rate of the process~\cite[Ch.~3]{coverElementsInformationTheory1991}.
Figure~\ref{fig:typical_sets} shows an example of the typical set and atypical sets over the probability measure of the observations.
\citet[Ch.~16.8]{coverElementsInformationTheory1991}
proved that typical sets exist for stationary ergodic stochastic processes.

\begin{figure*}[t]
    \centering
    \subcaptionbox{
        \small
        $p>q \!~\cup~\! |\Theta| > 1 \!\!\implies\!\! \exists$ undetermined set.
        \label{fig:typical_sets:neq}
    }[.4\textwidth]{
        \centering
        \includegraphics[width=.9\linewidth]{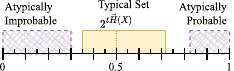}
    }
    \hfill
    \subcaptionbox{
        \small
        Typical set thresholds over total observations.
        \label{fig:typical_sets:plot}
    }[.592\textwidth]{
        \centering
        \includegraphics[width=.9\linewidth]{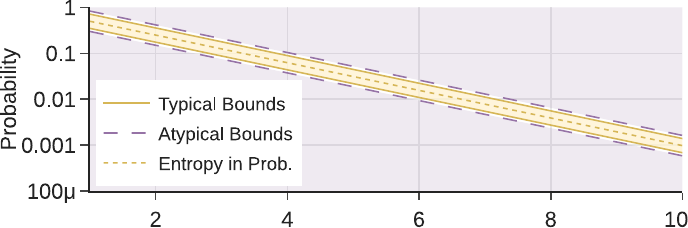}
    }
    \caption{
        \small
        The sample complexity's typical set thresholds for a single hypothesis of a fair coin with $p=0.7$, $q=0.6$ over $P(\vec{X}^t_1 = \vec{x}^t_1)$.
        Figure~\ref{fig:typical_sets:neq} depicts the typical and atypical sets at $t=1$ bounding the block entropy in probability space. 
        Figure~\ref{fig:typical_sets:plot} shows the bounds over 10 observations with log scaled probability.
        Observation sequences within the typical set support the hypothesis.
        $p$ and $q$ respectively determine the accepted probability of verifying a hypothesis and rejecting all hypotheses.
        If $p > q$ or $|\Theta| > 1$, then there exists an undertermined set of observations where more samples are required to determine set membership.
    }
    \label{fig:typical_sets}
\end{figure*}

\begin{definition}
\label{def:sample_complexity}
    For stationary ergodic processes with information dissimilarity measure $D$ and an accepted measure error $\epsilon_D \in \mathbb{K}, 0 \le \epsilon_D < \infty$,
    the \textbf{sample complexity}
    $\mathbf{i}(
        \vvec{\mathbf{X}}^*_1,
        \vvec{x}^t_1,
        p, 
        q, 
        \epsilon_D, 
        r
        ) \mapsto \N$
    is the number of query string symbols $\vvec{x}^t_1 \in \MeasurableInput^* : \ell(\vvec{x}^t_1) \le -\log_2(r), r \in \K, 0 \le r \le 1$
    observed when one of the following mutually exclusive cases is first satisfied:
    \begin{itemize}
        \item 
        verified set membership with probability $p$ where 
        $\exists(\vvec{X}^*_1 \in \vvec{\mathbf{X}}^*_t, \psi \in \vvec{X}^*_1) : D(\psi, \vvec{x}^t_1) \le \epsilon_D$
            and
            $\exists (\vvec{X}^*_1 \in \vvec{\mathbf{X}}^*_1):
                2^{t\entropyRate(\vvec{X}) + \log_2(p)}\le P(\vvec{X}^*_1 = \vvec{x}^t_1) \le 2^{t\entropyRate(\vvec{X}) - \log_2(p)}
                , p \in \K, 0 \le p \le 1$,
        \item 
        falsified set membership with probability $q$ where
        $\forall(\vvec{X}^*_1 \in \vvec{\mathbf{X}}^*_t, \psi \in \vvec{X}^*_1) : D(\psi, \vvec{x}^t_1) > \epsilon_D$
        or
        $\forall (\vvec{X}^*_1 \in \vvec{\mathbf{X}}^*_1):
            \big(
            P(\vvec{X}^*_1 = \vvec{x}^t_1) \le 2^{t\entropyRate(\vvec{X}) + \log_2(q)}
            \big) \cup \big( 2^{t\entropyRate(\vvec{X}) - \log_2(q)} \le P(\vvec{X}^*_1 = \vvec{x}^t_1) \big)
            , q \in \K, 0 \le q \le p$,
    \end{itemize}
\end{definition}
\begin{remark}
An alternative name could be the ``\textit{identification complexity}'' as it is fundamentally about how many observations are required in order to identify whether the observations belong to an element in the hypothesis set; however, we use the common term for consistency with prior works.
Also, the set of stochastic processes $\vvec{\mathbf{X}}^*_1$ may result in the sample complexity as a distribution over $\N$ as discussed in the following sections and generalizes the deterministic case of a set of strings $\MeasurableInput^*$.
\end{remark}

If $p=0$ or $\epsilon_D \rightarrow \infty$, then any model may be used without any observations as certainty and correctness are not of concern, respectively,
and accepting $r = 1 \implies$ accepting $p=0$ and the possibility $\epsilon_D \rightarrow \infty$.
$p = q =1$ requires absolute certainty. 
$\epsilon_D = 0$ requires absolute correctness.
Both imply only non-partial identification is accepted.
These default values, along with unbounded observations $-\log_2(r) \rightarrow \infty$ when $r = 0$, may result in a nonhalting program that attempts to compute the sample complexity.
We are most interested in when a finite sample complexity occurs without enforcing a maximum with $r > 0$, and so we consider the default value $r=0$.
Probability $q$ is for probable exhaustive falsification when we cannot have absolute certainty.
We refer specifically to the sample complexity or its properties with at least probability $p = q = 1$, and at most error $\epsilon_D = 0$ with respect to the \citet{hammingErrorDetectingError1950} distance up until asymptotic sample complexity in Theorem~\ref{th:asymp_sample_complexity}, where after we use the relative entropy.
\section{Identifying Information from Direct Observations}
\label{sec:id_direct}
We now prove properties of identification and sample complexity
    when the set of known hypotheses $\mathbf{\Theta}$ is also the sample space of the observed string $\theta$.
In this case, there is no hidden or latent state space as the hypothesis space \textit{is} the observed sample space.
Note that $\mathbf{\Theta}$ is a set meaning all the strings are unique; however, in order to compute set membership, the comparison of the index-aligned symbols in the strings requires there to be a chosen order of comparison.
Every set that is computable is fully identifiable when the strings in the set are what is observed, as we prove here and later, but they may not be fully identifiable when what is observed is the output of a computable function of the set's strings.
Given computation is our concern, these sets will often be manipulated as sequences of their elements.

We assume without loss of generality that strings are parsed from left-to-right.
A set of strings parsed in such a consistent way can be represented as a prefix-tree, analogous to a trie data structure~\citep{knuthArtComputerProgramming1998}.
This is known as a ``context-tree model'' as introduced by \citet{rissanenModelingShortestData1978,rissanenUniversalDataCompression1983}.
Algorithm~\ref{alg:id_sorted_set} computes $\mathbf{i}(\mathbf{\Theta}, \theta, 1, 1, 0, r)$ as the modified indicator function over a sorted set of binary strings that corresponds to a context-tree model,
where $p$ and $q$ are not able to be changed as parameters of the algorithm.
Depth and breadth first search variants are possible without assuming $\mathbf{\Theta}$ is sorted.
Algorithm~\ref{alg:id_sorted_set} is useful for explaining the nuance of identification with absolute certainty and serves as a starting point from which we build upon the more general computations of identification.
    
\begin{algorithm}[t]
\caption{Modified Indicator Function to Identify a Binary String using a Sorted Set}
\label{alg:id_sorted_set}
    \begin{algorithmic}[1]
    \Procedure{Identify}{$\mathbf{\Theta}, \vvec{\theta}, r$} \Comment{sorted set, query string, \& resolution}
    \State $i \gets 1$ \Comment{The information to identify if in or out of set is at least one bit}
    \State $j \gets 1$ \Comment{The index or number of elements checked}
    \State $\vvec{\psi} \gets \mathbf{\Theta}_j$
    \For{$i \gets 1; i \le \ell(\vvec{\theta}) \And i \le -\log_2(r); i \gets i + 1$} \Comment{Incrementally observe bits of $\vvec{\theta}$}
        \State $h \gets j$
        \While{$\vvec{\psi}^{~\!\!i}_1 < \vvec{\theta}^{~\!\!i}_1$} \Comment{Falsify those less than in value}
            \State $j \gets j + 1$
            \State $\vvec{\psi} \gets \mathbf{\Theta}_j$
            \If{$j > |\mathbf{\Theta}| ~||~ \vvec{\psi}^{~\!\!i}_1 > \vvec{\theta}^{~\!\!i}_1$} \Comment{Exhaustively falsified the set's elements}
                \State \textbf{return} $(0, h, i)$
            \EndIf
        \EndWhile
    \EndFor 
    \State \textbf{return} $(1 , h, i)$ \Comment{Verified in set.}
    \EndProcedure
    \Comment{\!Halts $\!\iff\!$ finite subset $\{\vvec{\psi} \in \mathbf{\Theta}\}$: each sharing a finite prefix with $\vvec{\theta}$}
    \end{algorithmic}
\end{algorithm}

\begin{theorem}
\label{th:id_verify}
    (Verify a String's Set Membership by Exhaustive Symbol Comparisons)
    Pairwise equality of a string $\vvec{\theta}$ to another $\vvec{\psi} \in \mathbf{\Theta}$, and thus verification that $\vvec{\theta} \in \mathbf{\Theta}$, can only occur once all symbol pairs are observed and measured as equivalent between $\theta$ and $\vvec{\psi}$.
    \begin{equation*}
        \exists (\vvec{\psi} \in \mathbf{\Theta}) : \big(
        \forall (j \in \{1, 2, ..., \ell(\vvec{\psi})\}) : \vvec{\psi}_j = \vvec{\theta}_j
        \big)
    \end{equation*}
\end{theorem}
\begin{proof}
    Let there be a binary string $\vvec{\theta}$ in question of belonging to a set of binary strings $\mathbf{\Theta}$.
    To verify or falsify the equality of $\vvec{\theta}$ to each $\vvec{\psi} \in \mathbf{\Theta}$ requires a comparison of their index-aligned symbols and $-\log_2(r) \ge \ell(\vvec{\theta})$.
    Otherwise, when $-\log_2(r) < \ell(\vvec{\theta})$ only partial identification is possible, resulting in a subset of $\mathbf{\Theta}$ whose prefixes match $\vvec{\theta}^{-\log_2(r)}_1$, but no verification with certainty.
    If the lengths do not match, then $\vvec{\theta} \neq \vvec{\psi}$.
    If the strings are prefixed with their length, then this is determined early on by checking first the binary representation of their lengths, where since zero length strings would not be included then a length of ``zero'' could be used to inform an infinitely long string.
    Otherwise, the number of symbol comparisons is as few as the length of their shared prefix to a maximum of as many as $\min(\ell(\vvec{\theta}), \ell(\vvec{\psi}))$ for that string pair, assuming a stopping symbol is encoded in the binary strings.
    Stopping information must be assumed available otherwise we encounter the halting problem.
    Once a different symbol between the strings is observed, that pair's equality is falsified.
    Only once a matching string is found in $\mathbf{\Theta}$, Algorithm~\ref{alg:id_sorted_set} will halt with verification.
    Verification requires that the query string be finite in length in order for the Turing machine to determine pairwise string equality and halt.
    This also requires that $|\mathbf{\Theta}| < \infty$ or that there be no countably infinite subsets that must be checked prior to the comparing the correct string.
    Otherwise, the machine will not halt and identification remains in an undetermined state.
\end{proof}

\begin{theorem}
\label{th:id_falsify}
    (Verify a String is Not in the Set by Exhaustive Falsification of Strings)
    An observed string $\vvec{\theta}$ can only be determined to not belong to a set of strings $\mathbf{\Theta}$ if all the possible strings are determined to not equal $\vvec{\theta}$ due to at least one differing pair of symbols.
    \begin{equation*}
        \forall (\vvec{\psi} \in \mathbf{\Theta}): \big(
            \exists (j \in \{1, 2, ..., \ell(\vvec{\psi})\} 
            ): \vvec{\psi}_j \neq \vvec{\theta}_j \big)
    \end{equation*}
\end{theorem}
\begin{proof}
    Continuing from the proof of Theorem~\ref{th:id_verify}, if $\forall(\vvec{\psi} \in \mathbf{\Theta}): \vvec{\theta} \neq \vvec{\psi}$, then the machine computing  Algorithm~\ref{alg:id_sorted_set} will halt if the subset of strings compared is finite and return the result of exhaustive falsification. 
    If $|\mathbf{\Theta}| \rightarrow \infty$, then the machine may never halt. 
    The machine will only halt if there is a structure or information about the content within the set, such as a context-tree or sorted set, in which case subsets of $\mathbf{\Theta}$ do no need compared to determine the query string's set membership due to comparing the strings in ascending order.
    For Algorithm~\ref{alg:id_sorted_set} or a context-tree, the order of comparisons can skip infinitely long strings by checking shortest strings first.
    However, if not using Algorithm~\ref{alg:id_sorted_set} or a context-tree and if a countably infinite subset must be checked given the order of string comparisons, then the machine will not halt and identification will remain in an undertermined state.
    If both strings being compared are countably infinite in length but are different, then there is a finite prefix of equivalent symbols that is checked until the first identifying symbol difference, thus this program only never halts if there is a countably infinite subset of strings to compare against to determine $\vvec{\theta} \notin \mathbf{\Theta}$.
\end{proof}
\begin{remark}
A set of arbitrary length strings can always be sorted in finite time because their matching prefixes that need to be checked are ensured to be finite in length due to unique items in a set.
\end{remark}

Notice Algorithm~\ref{alg:id_sorted_set}'s parameter $r$, which when provided determines the maximum observations $-\log_2(r)$, i.e., the maximum length of string comparisons considered.
If the string matches the prefix $\vvec{\theta}^{-\log_2(r)}_1$, then it is considered as verified in the set.
However, it is probably more desirable to instead return the subset of strings that all match the prefix if there exists more than one as they then form a partially identified subset.
Any of those partially identified strings could be a match to the string in question and further observations would be necessary to identify them from one another.

\begin{theorem}
    For all partially identifiable subsets $\mathbf{\Psi} \subseteq \mathbf{\Theta}$ there exists a computable function $f$ such that $f(\mathbf{\Theta}) = \mathbf{\Psi}$.

\end{theorem}
\begin{proof}
    This follows by Definition~\ref{def:partial_id} where nuance in detail is whether the program halts or not.
    Assuming $\vvec{\theta}^t_1 : t < -\log_2(r) \le \ell(\vvec{\theta})$
        then only partial identification is possible.
    If the program halts, which means only a finite subset of strings are compared, Algorithm~\ref{alg:id_sorted_set} could be modified to return a list of string indices that had a matching prefix $\mathbf{\Psi} \subseteq \mathbf{\Theta}: \vvec{\theta} = \vvec{\psi}$, which is the partially identified subset.
    These strings are only determined to match up to length $t$ and identification with certainty is not possible.
    If the program does not halt, it may be modified to include side effects or some output tape where it writes the indices of falsified strings or indicates its location in the sorted set or context tree to inform what is being partially identified and what was falsified.
\end{proof}
\begin{remark}
An index set of matching $t$ length prefixes for $\mathbf{\Psi} \subseteq \mathbf{\Theta}$ may be provided this way.
\end{remark}

\begin{corollary}
    A countably infinite string cannot be verified in a finite amount of observations,
    however it can be partially identified in finite observations.
\end{corollary}
\begin{proof}
    If $\vvec{\theta}^t_1 : t \rightarrow \infty$ and $t > -\log_2(r): r > 0$, 
        then if a matching prefix occurs, it may be partially verified $\vvec{\theta}^{-\log_2(r)}_1 = \vvec{\psi}^{-\log_2(r)}_1 : \vvec{\psi} \in \mathbf{\Theta}$.
    If only a finite subset of string comparisons are necessary to encounter the matching string, then verification of set membership may be determined.
    If only a finite subset of comparisons are necessary to exhaustively falsify, then falsification of set membership may be determined.
    If $r = 0$, then the program will not halt due to countably infinite equivalent symbols to compare.
\end{proof}

\begin{theorem}
\label{th:exact_sample_complexity}
    (Exact Sample Complexity)
    The sample complexity is exact,
    i.e.,
        surely $\mathbf{i}(\mathbf{\Theta}, \vvec{\theta}, 1, 1, 0, 0) = c$
    for some finite constant $c$,
    if and only if
    the following statements are true
    \begin{itemize}[noitemsep,nolistsep]
        \item there is an accepted encoding of the hypotheses' descriptions and the observations to a shared string representation space
        \item an identification program halts for that encoding given $\mathbf{\Theta}$ and $\vvec{\theta}$
   \end{itemize}
\end{theorem}
\begin{proof}
    $(\implies)$: If the sample complexity is exactly $c$, that means there exists a constant number $c= \mathbf{i}(\mathbf{\Theta}, \vvec{\theta}, 1, 1, 0, 0) < \infty$ of observations that identify $\vvec{\theta}$ with respect to $\mathbf{\Theta}$.
    The identification program must halt for that encoding, otherwise the sample complexity would not be finite and identification would be undetermined.
    If the sample complexity could not be $c$, then there is not an encoding for which $c$ observations results in a halted program.
    
    $(\impliedby)$: If there exists an identification algorithm that halts for the given encoding's symbol order, then only a finite amount of symbols in the query string were observed during that algorithm's computation, and so the sample complexity is that finite amount of observations.
    If there is no ordering of symbol comparisons that results in $c$ observations and then halting, then the sample complexity could not be $c$.
\end{proof}

There may be multiple possible encodings that are all equally correct where they contain the same information, but have a different enumeration that determines the order of the symbols in the strings.
This different ordering effects the sample complexity, and then we must consider the distribution of possible encodings.

\begin{theorem}
\label{th:surely_finite_id_info}
    (Surely Finite Sample Complexity)
    The sample complexity of a string is surely finite if and only if there exists an identification program that halts for all orderings of the string's symbol comparisons.
\end{theorem}
\begin{proof}
    ($\impliedby$):
    Let there be an identification program that halts for all orderings of the symbols.
    The Turing machine that computes that program observes at most a finite amount of symbols of the query string's prefix prior to determining if it is in the set or not and then halting.
    Therefore, the sample complexity of the string is surely finite.
    
    ($\implies$):
    Let there be a string with surely finite sample complexity.
    By our definition of sample complexity, this means that the string can always be identified after the comparison of a finite number of symbols.
    The set of all possible sequential symbol observations before identification is then finite, and like all finite sets can be enumerated by a Turing machine.
    Therefore, there exists an identification program that halts for all orderings of the string's symbol comparisons.
\end{proof}
\begin{remark}
This relates to computability of the indicator function, which we address further in Section~\ref{sec:id_comp}.
\end{remark}
    
    We iterate through each case of the properties of $\mathbf{\Theta}$ and $\theta$ when given as input to $\mathbf{i}(\mathbf{\Theta}, \theta, 1, 1, 0, 0)$ to demonstrate when the sample complexity is either finite or infinite, both with absolute certainty.
    In the following cases, exhaustive falsification is guaranteed to be possible if $\vvec{\theta} \notin \mathbf{\Theta}$ because $|\mathbf{\Theta}| < \infty$ and $\big(\vvec{\theta} \in \mathbf{\Theta} : \ell(\vvec{\theta}) < \infty;$ or $\vvec{\theta} \notin \mathbf{\Theta}\big)$, 
        and thus
        the identification program halts
        for each case regardless of the symbol comparisons' order.
        However, verification is not guaranteed unless $\vvec{\theta} \in \mathbf{\Theta}$ and $\ell(\vvec{\theta}) < \infty$.
    \begin{itemize}[noitemsep,nolistsep]
        \item If $|\mathbf{\Theta}| < \infty ~\&~ \ell(\vvec{\theta}) < \infty ~\&~ \forall (\vvec{\psi} \in \mathbf{\Theta}) : \ell(\vvec{\psi}) < \infty$
            $\implies \mathbf{i}(\mathbf{\Theta}, \vvec{\theta}, 1, 1, 0, 0) < \infty$.
        \item If $|\mathbf{\Theta}| < \infty ~\&~ \ell(\vvec{\theta}) < \infty ~\&~ \forall (\vvec{\psi} \in \mathbf{\Theta}) : \ell(\vvec{\psi}) \rightarrow \infty$
            $\implies \mathbf{i}(\mathbf{\Theta}, \vvec{\theta}, 1, 1, 0, 0) < \infty ~\&~ \vvec{\theta} \notin \mathbf{\Theta}$.
        \item If $|\mathbf{\Theta}| < \infty ~\&~ \ell(\vvec{\theta}) \rightarrow \infty ~\&~ \forall (\vvec{\psi} \in \mathbf{\Theta}) : \ell(\vvec{\psi})  < \infty$
            $\implies \mathbf{i}(\mathbf{\Theta}, \vvec{\theta}, 1, 1, 0, 0) < \infty ~\&~ \vvec{\theta} \notin \mathbf{\Theta}$.
    \end{itemize}
    If $\vvec{\theta} \in \mathbf{\Theta} $ and $ \ell(\vvec{\theta}) \rightarrow \infty$, then if the program reaches $\vvec{\psi} \in \mathbf{\Theta} : \vvec{\psi} = \vvec{\theta}$, the program will not halt due to infinite equivalent symbol comparisons and then $\mathbf{i}(\mathbf{\Theta}, \vvec{\theta}, 1, 1, 0, 0) \rightarrow \infty$.
    
    If $|\mathbf{\Theta}| \rightarrow \infty$ and $\mathbf{\Theta}$ is sorted,
        then the identification Algorithm~\ref{alg:id_sorted_set} will halt if
            $\ell(\vvec{\theta}^{-\log_2(r)}_1) < \infty$
            or
            $\vvec{\theta}^{-\log_2(r)}_1 \notin \mathbf{\Theta}$
            because both cases require only a finite number of comparisons of finite string prefixes, and thus results in a surely finite sample complexity regardless of the order of the symbols in the strings, albeit this is only possible due to the ordered set.
            
    If $|\mathbf{\Theta}| \rightarrow \infty$ and $\mathbf{\Theta}$ is unsorted, halting is not guaranteed and Algorithm~\ref{alg:id_sorted_set} cannot be used.
        Breadth-first traversal of $\mathbf{\Theta}$ will never halt, even if $\vvec{\theta} \in \!\mathbf{\Theta}$ and $\ell(\vvec{\theta}) \!< \!\infty$.
        The depth-first search on an unsorted set in Algorithm~\ref{alg:id_depth_first_set} will only halt if
        $\ell(\vvec{\theta}^{-\log_2(r)}_1) < \infty$, $\vvec{\theta}^{-\log_2(r)}_1 \in \mathbf{\Theta}$, and the ordering of comparisons happens to be such that the finite verification occurs before the subset of infinite to-be falsified strings is encountered, which once encountered would result in a program that does not halt. 
        
    If one can assume a context-tree model exists for the possibly countably infinite strings contained within it, then as each action is to only walk the tree in finite steps, rather than have to iterate through the possible strings that would be falsified, then the only non-halting case is the case of a countably infinite matching prefix.
    However, such a context-tree model or sorted set cannot be constructed within finite time due to the infinite strings.

\begin{corollary}
    If the sample complexity of a given set is surely finite for all query strings, then the set is fully identifiable.
\end{corollary}
\begin{proof}
    If the sample complexity is surely finite, then for any ordering of the strings, the sample complexity must be finite, and furthermore, that implies that for all query strings the sample complexity is finite and thus the set must be fully identifiable as it is able to compute the verification and falsification of any query string in or out of the set and then halt.
\end{proof}
    
    We briefly address uniformly random symbol comparisons, which is the one case where the sample complexity is a nondegenerate probability distribution even if the strings' entire information is directly observable, i.e., finite and information of their ends known.
    This is also useful to consider the probable sample complexity if the best encoding is unknown and one were selected at random.

\begin{theorem}
\label{th:id_pair_finite_unordered}
    (The Sample Complexity Distribution
    to Pairwise Identify
    Finite Strings)
    When the order of symbol comparisons between a pair of finite $L$ length strings is uniformly random without resampling,
    the sample complexity will vary over the possible orderings due to identification by falsification resulting in
    the following probability distribution over $i \in [1, L-K]$ with probability zero everywhere else, where $K$ is the number of different symbol pairs,
    \begin{equation}
    \label{eq:id_unordered_finite_pdf}
        P\Big(\mathbf{i}\big(\{\psi^L_1\}, \theta^L_1, 1, 1, 0, 0 \big) =i\Big) =
            \frac{(L-K)!}{L!}
            \left(\frac{(L-i+1)!}{(L-K-i+1)!} - \frac{(L-i)!}{(L-K-i)!}\right)
    \end{equation}
    and the corresponding cummulative probability distribution:
    \begin{equation}
    \label{eq:id_unordered_finite_cdf}
        P\Big(\mathbf{i}\big(\{\psi^L_1\}, \theta^L_1, 1, 1, 0, 0 \big) \leq i\Big) = 1 - \frac{(L-K)!(L-i)!}{L!(L-K-i)!}
    \end{equation}
\end{theorem}
\begin{proof}
    Let there be two finite $L$ length strings ${\theta}^L_1$ and ${\psi}^L_1$ with
    a uniformly random ordering of comparisons for their indexed-aligned symbol pairs $({\theta}_j, {\psi}_j): j \sim U(1, L)$.
    For the first comparison, if $\theta_j \neq \psi_j$, then the two strings' equality is falsified and the sample complexity is $1$.
    Otherwise, $\theta_j = \psi_j$, and the process continues without replacement on the remainder of the symbols, where $j \sim U(1, L-1)$.

    Let $\mathbf{i}(\{\psi^L_1\}, \theta^L_1) = \mathbf{i}(\{\psi^L_1\}, \theta^L_1, 1, 1, 0, 0)$.
    The sample complexity $\mathbf{i}(\{\psi^L_1\}, \theta^L_1)$ is dependent upon the total number of symbol pairs $L$, and the number of different symbols between the strings given by Hamming distance $K = h(\theta^L_1, \psi^L_1)$.
    Falsification occurs as soon as one of the different symbol pairs are compared, and so a maximum of $L-K+1$ if $K \ge 1$, otherwise $L$ comparisons occur for verification, thus $\mathbf{i}(\{\psi^L_1\}, \theta^L_1) \in [1, L]$.
    This sampling of pairs results in the sample complexity
    being a discrete chain of samples from Bernoulli distributions whose total possibilities $L$ decreases by the $i$ comparisons made, which in turn increases the probability of encountering a pair different symbols $p_i = \frac{K}{L-i}$, if any exist.
    
    The probability of partially verifying a string after $i$ comparisons is the hypergeometric probability:
    \begin{equation}
    \label{eq:partial_verify_hypergeo}
        \frac{{K \choose 0}{L - K \choose i}}{{L \choose i}} = \frac{{L - K \choose i}}{{L \choose i}} = \frac{(L-K))!}{i!(L - K - i)!} \frac{i!(L-i)!}{L!}
    \end{equation}
    By the law of total probability, the probability of falsifying the string in up to $n$ observations is $1$ minus Equation~\ref{eq:partial_verify_hypergeo}, which results in the cumulative probability distribution in Equation~\ref{eq:id_unordered_finite_cdf} for the sample complexity due to falsification where $1 \le i \le L-K$.
    The probability distribution is found by taking the first difference:
    \begin{align*}
        P\Big(\mathbf{i}\big(\{\psi^L_1\}, \theta^L_1 \big) = i\Big) &= P\Big(\mathbf{i}\big(\{\psi^L_1\}, \theta^L_1 \big) \leq i\Big) - P\Big(\mathbf{i}\big(\{\psi^L_1\}, \theta^L_1 \big) \leq (i - 1) \Big) \\
        &= \frac{(L-K)! (L - i + 1)!}{L!(L - K - i + 1)!} - \frac{(L-K)! (L-i)!}{(L-K-i)! L!}\\
        &= \frac{(L-K)!}{L!} \left( \frac{(L - i + 1)!}{(L - K - i + 1)!} - \frac{(L-i)!}{(L-K-i)!} \right)
    \end{align*}
\end{proof}
    
If the two strings do not contain different symbols, then there is no identifying information and so once exhausted the two strings will be determined as equal.
If they are a substring of longer strings, then the result is partial identification.
The sample complexity distribution to pairwise identify infinte strings is binomial and is addressed in Appendix~\ref{sec:app:id_inf_direct_obs}.

\subsection{Identifying Novelty from Direct Observations}
\label{sec:id_nov_direct}

In the prior cases covered, any query string whose set membership is falsified is unknown to that set, as discussed in Section~\ref{sec:what_is_novelty}.
To extend this further, we can consider taking slices across the strings by their indices.
A single index can form a symbol set which can test the set membership of the query string's indexed symbol.
Furthermore, and more interestingly, we can consider a range of symbols to form a set of substrings where
$\mathbf{i}\big( \{\vvec{\psi}_j^L\}, \vvec{\theta}^L_j, 1, 1, 0, 0 \big) : L, j \in \big[1, \max_{\vvec{\psi} \in \mathbf{\Theta}}(\ell(\vvec{\psi})) \big)$
and run the same Algorithm~\ref{alg:id_sorted_set} on the set of those substrings to determine if the query string's shared indexed substring is in the set (known) or not (novel).
Viewed this way enables considering the unique information of properties or components of the query string to the strings in $\mathbf{\Theta}$.
If considering noncontiguous symbols in the strings, then this opens up a combinatorial explosion of possible subsets to explore, however this is a way to consider novelty with regard to a subset of properties represented within the strings.
Considering a growing set, add every falsified query string to the known set in its proper location, if a sorted set.

\section{Identification and Computability}
\label{sec:id_comp}
We have covered computing the identification of a string's membership in a set of strings when the shared representation's symbol indices are known, which sets up discussing symbol changes through computable functions of those strings.
We first establish the connection of identifiable elements to
    computable
    sets and strings,
    and the nuance in whether the Turing Machine halts or not in their enumeration or generated output.

\begin{theorem}
    A set is computable if and only if the set forms a set of strings that is fully identifiable and for all query strings has a surely finite sample complexity for all symbol comparison orderings.
\end{theorem}
\begin{proof}
    ($\implies$):
    If a set is computable then the indicator function is able to be computed by a Turing Machine for any query string and halt in finite observations.
    This also holds for all subsets of the computable set,
    which makes
    the set fully identifiable with a surely finite sample complexity.
    
    ($\impliedby$):
    A set is fully identifiable if all of the set's strings can be differentiated from one another through their respective observations and also differentiated from other strings not in the set, and thus may be given a unique label by a Turing machine that halts.
    Such a set's sample complexity is surely finite for any query string if the set is finite, which means the identification algorithm will halt and determine any query string's set membership.
    Therefore, the fully identifiable set that is surely finite for any symbol ordering is a computable set.
\end{proof}
\begin{remark}
This becomes straightforward due to a mathematical set always consisting of unique elements, and thus subsets of a computable set are always computable, and so is the same for a fully identifiable set with a surely finite sample complexity.
Our consideration of a fully identifiable set with surely finite sample complexity from a computation standpoint becomes relevant later when what is observed is no longer directly the symbols of the strings being compared but are the output of a Turing Machine given the description of the strings.
\end{remark}

Another way to phrase a fully identifiable set is that the set can be enumerated and halt for each individual string in the set.
Every finite length binary string is computable, while the infinite length strings are c.e., but not necessarily computable.
In the following as we discuss the sample complexity, identification's relationship to computable and c.e. strings is made clear.

\begin{theorem}
    If a set is c.e., then for all strings in that set, the indicator function halts and are thus identifiable, specifically by verification of matching a string in the set, and only halts due to falsification for some query strings, if any.
\end{theorem}
\begin{proof}
    Let there be a c.e. set $\mathbf{\Theta}$.
    By definition of c.e., all $\vvec{\psi} \in \mathbf{\Theta}$ when given to a machine computing the indicator function will be determined to be in the set and the machine will halt.
    Because it halts, $\forall(\vvec{\psi} \in \mathbf{\Theta}): \mathbf{i}(\mathbf{\Theta}, \vvec{\psi}, 1, 1, 0 , 0) < \infty$, and is thus identifiable by verification.
\end{proof}
\begin{remark}
    On the properties of such sets: $\forall(\vvec{\psi} \in \mathbf{\Theta}): \ell(\vvec{\psi}) < \infty$, otherwise verification would be unable to be determined.
    If $|\mathbf{\Theta}| \rightarrow \infty$, then if $\vvec{\theta} \in \mathbf{\Theta}$ be verified and the machine will halt if any order is chosen due to finite prefix of an infinite set.
    If the set's strings to be compared are randomly sampled then it will almost surely not halt as in Theorem~\ref{th:id_pair_infinite_unordered}.
\end{remark}
\begin{corollary}
\label{th:co-ce_id}
    If the complement $\bar{\mathbf{\Theta}}$ of a set $\mathbf{\Theta}$ is c.e., then for all strings not in the set $\mathbf{\Theta}$, the indicator function halts and are thus identifiable, specifically by exhaustive falsification of equality to all strings in $\mathbf{\Theta}$.
\end{corollary}
\begin{proof}
    If $\bar{\mathbf{\Theta}}$ is c.e., then $\forall(\vvec{\theta} \notin \mathbf{\Theta}, \vvec{\psi} \in \mathbf{\Theta}): \vvec{\theta} \neq \vvec{\psi}$ can be determined by a machine that halts, and thus the set is always identifiable by exhaustive falsification for all query strings not in $\mathbf{\Theta}$.
\end{proof}
\begin{remark}
    When $|\mathbf{\Theta}| < \infty$, exhaustive falsification can occur in finite observations, regardless of query string length.
    Thus a finite sized sets' complement is always c.e. and is identifiable by exhaustive falsification.
    If $|\mathbf{\Theta}| \rightarrow \infty$ and $\vvec{\theta} \notin \mathbf{\Theta}$ then infinite falsification of string equality may occur unless the set's structure, e.g., string order, can be used to falsify the remaining subset without checking individual strings.
    Such sets' complement are not c.e. and not identifiable by exhaustive falsification.
\end{remark}

Theorem~\ref{th:co-ce_id} indicates that all that is able to be computed as unknown for a given known set is co-c.e.
A set of strings of countably infinite length does not always result in a co-c.e. set for any query string.

\begin{corollary}
    All computable numbers are partially identifiable up to a finite $r$ length prefix, but not all computable numbers are identifiable with absolute certainty.
\end{corollary}
\begin{proof}
    The computable numbers described by finite length strings may be partially identified by their finite prefix among any set.
    The computable numbers described by a string of countably infinite length can have their finite $r$ length prefix computed by definition.
    Such computable numbers will never be verified with absolute certainty due to infinite length, and so their finite $r$ length prefix may partially identify them among any set.
\end{proof}
\section{Symbol Change: Finite Information Spread Over Longer Strings}
\label{sec:finite_over_infinite}
The theorems so far focused on
identifying information and the sample complexity of a binary string being directly observed one symbol at a time with respect to a set of binary strings with a known shared encoding.
In practice, we are not necessarily observing the exact strings we are considering in our hypothesis sets, unless the observations are their own minimum length description as with algorithmically random sequences.
Often, we have a separation of the observation space from the latent parameter set that forms our hypothesis set of possible data generating processes.
All of the above still applies, although what is compared is the 
possible outputs of a computable function given the latent parameters to the actual observations.
In this case, we consider the set of latent descriptions transformed through their computable sample function $\vvec{\mathbf{X}}^t_1 = \{\vvec{X}^t_1 = \lambda_S(\theta, t, Z) : \theta \in \mathbf{\Theta} \} $, where the output $\vvec{\mathbf{X}}^t_1$ is the set of possible  observable $t$-length sequences with a corresponding potentially random process for each description $\theta \in \mathbf{\Theta}$.
Soon to be described, $Z$ is an information source used only if the state is random.
The task becomes finding the mapping of the information contained within the observations back into the hypothesis set that possibly generated them.

From this point on, we focus on only finite sized parameter sets at moments in time with finite length strings as parameters, however, in following sections we consider how their observations finitely grow to infinite. 
That is, we consider a Turing Machine using binary strings with a prefix code read from an input tape for the description, read and written to a latent state tape, and read and written to an output tape, which is the only observed state.
Only the latent state tape erases information as the prior state becomes no longer necessary for the process.
In practice, only finite information may be considered for the description, internal state, and the observations.
Given this we consider how finitely described processes may spread their information over finitely growing observations from the perspective of a change of symbols.

\subsection{Finite Information Spread Over Infinite Observations}

We now address how programs may spread their finite information in their descriptions across sequences of countably infinite length and how this effects identification.
This connects to the classic, well-studied asymptotic statistics in probability theory.
First, we address how to compute samples of one independent random variable from samples of another, denoted as $\lambda_S(\theta, t, Z)$ finally making use of $Z$.

\begin{theorem}
\label{th:fair_coin_inv_cdf}
    (Compute Any Discrete Probability Distribution from  Flips of a Fair Coin)
    For any computable random function $f: \Z \mapsto \Z$ there exists a deterministic function $g: \Z \times \mathbb{B}^n \mapsto \Z$ with the property
   \begin{equation*}
        \forall \Big(f; x \in \Z, (n \ge r : n,r \in \Z^+), z^n \sim \text{Bernoulli}\Big( \frac{1}{2} \Big)^n \Big) ~ \exists \Big( g, r : P\big(f(x)\big) = P\big(g(x, z^n)\big) \Big)
   \end{equation*} 
\end{theorem}
\begin{proof}
    Let $P\big( f(x) \big)$ denote the distribution of the output of $f$ when $x$ is provided as input.
    The support of $P\big( f(x) \big)$ is some non-empty subset of $\Z$, and by computability of $f$ there must be some function $F(x,y)$ that computes the probability of $P\big( f(x) \ge y \big)$.
    The desired function $g$ can be directly constructed using $F$ along with bits sampled from a Bernoulli distribution with $p = \frac{1}{2}$. 

    Let $r$ be the number of bits required to represent the greatest common denominator of the range of $F$.
    If we sample $n \ge r$ bits from the Bernoulli distribution then the binary fraction represented by those bits will fall in the range
        $\big[ F(x,y-1), F(x,y) \big)$ with probability
        $F(x,y) - F(x, y - 1) = P\big( f(x) = y \big)$.

    Let $g(x,z^n) : z^n \sim \text{Bernoulli}(\frac{1}{2})^n$ be the function
    $$g(x,z^n): \min_{y}\big( F(x,y) \big) > z^n$$
    Then $P\big( g(x,z^n) \big) = P\Big( x \in \big[ F(x,y-1), F(x,y) \big), y \Big) = P\big( f(x) = y, y\big) = P\big(f(x)\big)$.
\end{proof}
\begin{remark}
    The sampling of
    any computable random variable $X$ may be computed from the
    description of its sample space,
    its probability measure,
    and an entropy source $Z$,
    such as the resulting bits of repeatedly tossing a fair coin provided by an oracle.
    The resulting finite bit sequence $\vvec{z}$ is treated as a binary fraction and grows until only one sample of $X$ corresponds to the binary fraction determined by $X$'s inverse cumulative distribution.
\end{remark}

Theorem~\ref{th:fair_coin_inv_cdf} 
is the fundamental concept and computation of the inverse cumulative probability function to get a sample that corresponds to a given probability.
The expected number of coin flips $\vvec{z}^n \sim Z$ for the optimal algorithm to compute one sample from $x \sim X$ is $H(X) \le n < H(X) +2$~\cite[Ch.~5.11]{coverElementsInformationTheory1991}.
This computation is similar to how hardware random number generators work, except the fair coin is replaced with some other physical source of entropy such as the noise in air pressure or temperature measurements native to the computer hardware as specified in~\citet{turanRecommendationEntropySources2018} and \citet{barkerRecommendationRandomBit2024}.
This is also similar to the \textit{reparameterization trick} used in computational approaches like variational Bayesian statistics, as discussed by \citet{kingmaAutoEncodingVariationalBayes2022a}.

\begin{theorem}
\label{th:finite_info_over_infinite}
    The information in a finite length string can be spread over an infinitely long string
        without redundancy
        if
        and only if
        there also is an additional
        infinite information source available.
\end{theorem}
\begin{proof}
    Let $\theta$ be a finite binary string and $\vvec{x}$ be a binary string with countably infinite length where each $t$-th bit is computed by the Turing machine $\lambda_S(\theta, 1, Z) = \vvec{x}_t$.
    
    Let $\lambda_S$ have access to an oracle $Z$ that provides the results of a fair coin flip.
    
    Let there be a mixture model of $|\mathbf{\Theta}|$ component distributions that are not equal to one another and thus form a fully identifiable set, where every unique symbol in $\vvec{\Theta}$ is aligned to one component, which in the case of binary strings for both $\theta$ and $\vvec{x}$ is two different Bernoulli distributions.
    
    Let the algorithm that spreads the information of $\theta$ over $\vvec{x}$ read $\theta$ one bit at a time and loop back to the beginning after reading the final bit of $\theta$.
    
    Let $\ell(\theta)$ and the index for each symbol be known to the observer.
    
    For each bit of $\theta$ read, the corresponding symbol's Bernoulli distribution is sampled as per Theorem~\ref{th:fair_coin_inv_cdf} to obtain a bit for the resulting $\vvec{x}_t$.
    
    ($\implies$)
    Given the observer knows everything but the symbols at each index and symbols' corresponding Bernoulli distributions,
        this may be treated as learning which component distribution is almost surely used at each index of $\theta$.
    Treating this as $\ell(\theta)$ different empirical processes updated as one cycles through the indices of $\theta$, we know by the Glivenko-Cantelli theorem that the empirical processes will asymptotically converge almost surely to their true distributions used at each index as the observations goes to infinite.
    The different distributions for the symbols are almost surely identifiable from one another as $\lim_{t\rightarrow\infty} \vvec{x}^t_1$ and thus the original bits of $\theta$ are also almost surely identified.

    ($\impliedby$)
    The other direction of the bijection is obtained from the \textit{data processing inequality}~\citet[Theorem~2.8.1]{coverElementsInformationTheory1991}.
    If the additional information source were only finite, then that string cannot be used to spread another finite string over an infinite length without redundancy.
\end{proof}
\begin{remark}
    Each Bernoulli distribution must have a probability distribution \textit{not} equal to any of the remaining $|\mathbf{\Theta}|-1$ components,
        otherwise there will be no means of identifying the component distributions apart from each other,
        thus hiding the enumeration of unique symbols ($0$ and $1$ for binary), which would
        hide the identifying information for the parameters in the observed string.
\end{remark}

    The asymptotic convergence also holds for the Bayesian case for any prior distribution over the parameters of the Bernoulli distribution proved by \citet{doobApplicationTheoryMartingales1949} with almost sure convergence guaranteed given the prior supports the true possible parameter for the Bernoulli distribution and because the space is countable as it is a computable set~\cite[Ch.~6.2]{ghosalFundamentalsNonparametricBayesian2017}.

    The above covers the case where the information of $\theta$ observed is strictly monotonically increasing with every observed symbol in $\vvec{x}$.
    In the relaxed monotonic increasing case where there can be no information of $\theta$ observed in a symbol of $\vvec{x}$, then one adds another component to the mixture model which stands as a random source of irrelevant noise with respect to observing $\theta$.

The relative entropy of each component distribution to the mixture of the remaining components serves as a dissimilarity measure of that component  to the others, and in this case, when the relative entropy is near zero the distributions are similar and more observations are necessary to identify those distributions from one another.
And given these are probability distributions, only in infinite observations can we be almost sure that the identifying information is observed, meaning that within finite observations certain identification will never be guaranteed.

\begin{corollary}
\label{th:asymp_sample_complexity}
    (Asymptotic Identification with an Infinite Sample Complexity)
    Let the information of an identifiable parameter $\theta \in \mathbf{\Theta}$ result in a finite sample complexity when the observation space is the parameter space.
    When that identifying information is spread over an infinitely long string by $\lambda_S(\theta, t, Z) = \vvec{x}^t_1$,
    the identifying information of $\theta$ will almost surely be observed in $\lim_{t \rightarrow \infty} \vvec{x}^{t}_1$
    as the number of observations $t$ goes to infinity.
\end{corollary}
\begin{proof}
    By Theorem~\ref{th:finite_info_over_infinite}, the finite identifying information and its finite sample complexity may be spread over an infinite string.
    As the number of observations goes to infinite, the identifying information is almost surely observed, which recovers Definition~\ref{def:id_param}.
\end{proof}

As mentioned, each component distributions' relative entropy to the rest affects the individual symbol's sample complexity and while we cannot recover the property of surely finite sample complexity, we can work with the sample complexity's distribution, including the distribution's properties such as the expected sample complexity or its other moments.
\section{Partial Identification and the Distribution of the Sample Complexity}
\label{sec:part-id_distrib_sc}
We addressed the various cases of when the sample complexity is finite and infinite for a deterministic observed process, and we addressed the asymptotic sample complexity when the observed process is random.
Now, we recover the finite sample complexity from the observations of a stochastic process at the cost of some uncertainty by accepting a level of probability $0 < p < 1$ that the observations could be generated by a known model, which
harkens back to statistical hypothesis testing, PAC learning sample complexity bounds, and property testing of distributions.
We start with an evaluator's perspective knowing the fixed ideal $\psi$ and knowing the predictor's prior distribution to compute the sample complexity distribution.
We then show how the predictor's perspective results in an estimate of the sample complexity based on the assumption that $\psi \in \mathbf{\Theta}$.
The predictor's estimate becomes closer to the correct sample complexity distribution given more observations.

In the following theorems, we obtain our results using Bayes Theorem.
Because we are concerned with a fixed finite set of hypotheses, any prior distribution can be used without risking an infinite sample complexity due to Doob's Theorem~\cite{doobApplicationTheoryMartingales1949}, which is further explained in~\cite[Ch.~6.2]{ghosalFundamentalsNonparametricBayesian2017}.
However, the choice of prior can change the sample complexity.
For a prior distribution over a fixed finite hypothesis set, we can safely use the \textit{principles of indifference}~\cite[Ch.~4]{keynesTreatiseProbability2010}, \textit{insufficient reasoning}~\cite[Ch.~3]{stiglerHistoryStatisticsMeasurement1986}, or \textit{maximum entropy}\cite{jaynesInformationTheoryStatistical1957,jaynesInformationTheoryStatistical1957a,jaynesPriorProbabilities1968}, which state to treat every hypothesis as equally probable without further evidence to the contrary.
This choice averages out the potential sample complexity distribution to not favor one over the other.

\begin{theorem}
\label{th:finite_id_info_prob}
    (Partial Identification of an I.I.D. Process with Probability $p$)
    Let the computable hypothesis set be $\mathbf{\Theta} : |\mathbf{\Theta}| < \infty, \forall(\theta \in \mathbf{\Theta}):(\ell(\theta) < \infty) $,
    the fixed identifiable parameter $\psi \in \mathbf{\Theta}$,
    $\vvec{z} \sim Z$ be the results of flipping a fair coin,
    $\lambda_S(\psi, t, \vvec{z}) = \vvec{x}^t_1$ be the observations from an i.i.d. process
    $\vvec{x} \sim X$,
    and
    the prior distribution be such that $\forall(\theta \in \mathbf{\Theta}):P(\vvec{\Theta}_0=\theta) < p$, where
    $0 < p < 1$ is the desired minimum probability that a model $\theta$
    generated the observations.
    Then,
    the
    sample complexity
    $\mathbf{i}(\vvec{\mathbf{X}}^t_1, \vvec{x}^{t}_1, p, 0, 0, 0)$
    forms a
    random variable over $\Z^+$
    that is determined by its
    moments,
    which are
    surely finite.
    The expected sample complexity is
    \begin{equation*}
        \underset{\vvec{x}_i \sim X_i 
        }{E}
    [ \mathbf{i}(\mathbf{X}^t_1, \vvec{x}^{t}_1, p, 0, 0, 0) ] =
        \frac{1}{H(X)} \Big( -I(\vvec{\Theta}_0 \!=\! \psi) + H_\otimes(X^t_1 || \hat{X}^t_1)  - \log(p) \Big) 
    \end{equation*}
\end{theorem}
\begin{proof}
    The stopping condition for the identification algorithm computing $\mathbf{i}(\mathbf{X}^t_1, \vvec{x}^{t}_1, p, 0, 0, 0)$ is defined as
    $P(\Theta = \psi | \hat{X}^{t}_1 = \vvec{x}^{t}_1) \ge p$.
    Thus, the sample complexity is a function of the posterior probability of any parameter meeting or surpassing $p$.
    Let $\vvec{\Theta}_0$ be an arbitrary prior distribution over the parameters $P(\vvec{\Theta}_0=\theta) > 0 : \theta \in \mathbf{\Theta}$.
    By Bayes rule the posterior distribution at time-step $t$ is 
    \begin{equation*}
        P(\vvec{\Theta}_t = \psi) = P(\vvec{\Theta}_{0} = \psi | \hat{X}^{t}_1=\vvec{x}^{t}_1) =
        \frac{
            P(\hat{X}^{t}_1 = \vvec{x}^{t}_1 | \vvec{\Theta}_0 = \psi)
            P(\vvec{\Theta}_0 = \psi)
        }{P(\hat{X}^{t}_1 = \vvec{x}^{t}_1)}
    \end{equation*}
    Let $\hat{X}_t$ be the resulting random variable from the posterior predictive distribution $P(\hat{X}_t=\vvec{x}_t | \vvec{\Theta}_t)$ at time-step $t$ where $\vvec{\Theta}_t$ defines the component probabilities of a mixture distribution over $\MeasurableInput$.
        The expected posterior probability $E \big[-\log P(\vvec{\Theta}_t = \psi) \big]$
        may be expressed as
    \begin{align*}
        \underset{\vvec{x}^t_1 \sim X^t_1}{E \big[} \mkern-15mu-\!\log P(\vvec{\Theta}_t \!=\! \psi) \big] \!=&
        -
        \sum_{\vvec{x}^t_1 \in \MeasurableInput^t} \mkern-5mu P(X^{t}_1  \!=\! \vvec{x}^{t}_{1}) \log 
        \frac{
            P(\hat{X}^{t}_1 = \vvec{x}^{t}_1 | \vvec{\Theta}_0 = \psi)
            P(\vvec{\Theta}_0 = \psi)
        }{P(\hat{X}^{t}_1 = \vvec{x}^{t}_1)} \\
        =& \sum_{\vvec{x}^t_1 \in \MeasurableInput^t} \!\!P(X^{t}_1 \!=\! \vvec{x}^{t}_{1}) \log 
        \frac
        {P(\hat{X}^{t}_1 = \vvec{x}^{t}_1)}
        {
            P(\hat{X}^{t}_1 = \vvec{x}^{t}_1 | \vvec{\Theta}_0 = \psi)
            P(\vvec{\Theta}_0 = \psi)
        }
    \end{align*}
    \begin{align*}
        =& \mkern-10mu \sum_{\vvec{x}^t_1 \in \MeasurableInput^t} \mkern-10mu P(X^{t}_1 \!=\! \vvec{x}^{t}_{1}) \mkern-3mu \log
        \mkern-5mu
        \Bigg(
        \frac
            {P(\hat{X}_{t} \!=\! \vvec{x}_{t} | \hat{X}^{t-1}_1 \!\!=\! \vvec{x}^{t-1}_1)}
            {
                P(\hat{X}_{t} \!=\! \vvec{x}_{t} | \hat{X}^{t-1}_1 \!\!=\! \vvec{x}^{t-1}_1, \vvec{\Theta}_{t-1} \!=\! \psi)
            }
        \frac
            {P(\hat{X}_{t-1} \!=\! \vvec{x}_{t-1} | \hat{X}^{t-2}_1 \!\!=\! \vvec{x}^{t-2}_1 )}
            {
                P(\hat{X}_{t-1} \!=\! \vvec{x}_{t-1} | \hat{X}^{t-2}_1 \!\!=\! \vvec{x}^{t-2}_1, \vvec{\Theta}_{t-2} \!=\! \psi)
            }
        \\
        & \hspace{260pt}
            \cdots
        \frac
            {P(\hat{X}_{1} \!=\! \vvec{x}_1)}
            {
                P(\hat{X}_1 \!=\! \vvec{x}_1 | \vvec{\Theta}_0 \!=\! \psi)
            }
        \frac
            {1}{P(\vvec{\Theta}_0 \!=\! \psi)}
            \Bigg)
    \end{align*}
    Then, by moving the prior out of the sum as $I(\Theta_0 \!=\! \psi)$, expanding the $P(\hat{X}_t = x_t)$ into the summation over the mixture, and
    recognizing that given $\psi \in \mathbf{\Theta}$ then $
        P(\hat{X}^t_1 \!=\! \vvec{x}^t_1 | \vvec{\Theta}_0 \!=\! \psi) \!=\! 
        P(X^t_1 \!=\! \vvec{x}^t_1) \!=\! 
        \prod^{t}_{i=1}
        P(X_i \!=\! \vvec{x}_i | X^{i-1}_1 \!\!=\! \vvec{x}^{i-1}_1)$,
    then
    the expectation is equal to
    \begin{align}
        &=
        I(\Theta_0 \!=\! \psi)
        +
        \mkern-5mu
        \sum_{\vvec{x}^t_1 \in \MeasurableInput}
        \mkern-5mu
        P(\vvec{X}^t_1 \!=\! \vvec{x}^t_1)
        \log
        \Bigg(
        \prod^{t}_{i=1}
        \frac
            {\underset{{\theta \in \mathbf{\Theta}}}{\sum}
                P(\hat{X}_i \!=\! \vvec{x}_i |
                \hat{X}^{i-1}_1 \!=\! \vvec{x}^{i-1}_1,
                \vvec{\Theta}_{i-1}\!=\!\theta)
                P(\vvec{\Theta}_{i-1}\!=\! \theta)
            }
            {
                P(X_{i} \!=\! \vvec{x}_i | X^{i-1}_1 \!=\! \vvec{x}^{i-1}_1)
            }
        \Bigg)
        \label{eq:id_with_prob_simplified}
        \\
        &= I(\vvec{\Theta}_0 = \psi) - 
        H_\otimes(\vvec{X}^t_1 || \hat{X}^t_1) + H(X^t_1)
        \nonumber \\
        &= I(\vvec{\Theta}_0 = \psi) - \KL(X^t_1 || \hat{X}^t_1)
        \label{eq:id_with_prob_brevity}
    \end{align}
    The sequence comparison in $\KL(X^{t}_1||\hat{X}^{t}_1)$ is important because time-step $t$ in $\lim_{t \rightarrow \infty}$ $\KL(X_t || \hat{X}_t) = 0$ given the countable set $\mathbf{\Theta}$, $\psi \in \mathbf{\Theta}$, and $P(\vvec{\Theta}_0=\psi)> 0$, then the Bayesian posterior converges almost surely as proven by \citet{doobApplicationTheoryMartingales1949}.
    $\lim_{t \rightarrow \infty}\KL(X^{t}_1||\hat{X}^{t}_1)$ captures the total information gain from the prior distribution to the converged posterior distribution and is finite.
    Recall from the tight entropy bounds on the algorithmic complexity 
        \cite[Th.~14.3.1]{coverElementsInformationTheory1991}
        that
        a finitely described computable i.i.d.
        random variable has finite entropy, thus upper bounding $\KL(\vvec{X}^{t}_1||\hat{X}^{t}_1) \le tH(\vvec{X}) <\infty$.
    When $E_{\vvec{x}^t_1 \sim \vvec{X}^t_1} \big[-\log P(\vvec{\Theta}_t = \psi) \big] = -\log(p)$ then the posterior distribution is expected to be $p$ probable yielding it partially identified
    with an expected finite number of observations equal to $t$ following from
    \begin{align*}
        I(\vvec{\Theta}_0 = \psi) - H_\otimes(X^t_1 || \hat{X}^t_1) + tH(X) &= -\log(p) \\
        I(\vvec{\Theta}_0 = \psi) - H_\otimes(X^t_1 || \hat{X}^t_1) &= -\log(p) - tH(X) \\
        \frac{1}{H(X)} \Big( -I(\vvec{\Theta}_0 = \psi) + H_\otimes(X^t_1 || \hat{X}^t_i)  - \log(p) \Big)  &= t
    \end{align*}
        The $m$-th raw moment of $-\log  P(\vvec{\Theta}_t = \psi)$ is
    \begin{align*}
        &\underset{\vvec{x}^t_1 \sim X^t_1}{E} \Big[ \big(-\log  P(\vvec{\Theta}_t = \psi)\big)^m \Big] =
            \Bigg(
                (-1)^{m+1} \Big( tH(X) + I(\vvec{\Theta}_0 = \psi) \Big)
                + (-1)^{m} H_\otimes(X^t || \hat{X}^t_1)
            \Bigg) \nonumber \\
            & \times 
            \Bigg(
                (-1)^{m-1}
                \Big(
                    I(\vvec{\Theta}_0 = \psi)^{m-1}
                    + \big( \sum_{\vvec{x}^t_1 \in \MeasurableInput^t} I(X^t = \vvec{x}^t_1) \big)^{m-1}
                \Big)
                + (-1)^m \Big(
                    \sum_{\vvec{x}^t_1 \in \MeasurableInput^t} I(\hat{X}^t_1 = \vvec{x}^t_1)
                \Big)^{m-1}
            \Bigg) \label{eq:id_mth_raw_logp}
    \end{align*}
    where
    $t$ can be isolated to get the sample complexity's $m$-th raw moment and all parts are finite for every $m$.
    Given that the sample complexity's support is non-negative and the raw-moments are finite, then the sample complexity's probability distribution is determined by its moments~\cite[Theorem~30.1]{billingsleyProbabilityMeasure1995}.
\end{proof}
\begin{remark}
This theorem focuses on the Bayesian posterior convergence to determine sample complexity.
Given $\psi \in \mathbf{\Theta}$, the posterior predictive distribution once converged will share typical sets with the ideal distribution, satisfying the probable verification case in Definition~\ref{def:sample_complexity}.
\end{remark}

\begin{corollary}
\label{th:pred_expect_with_prob}
    (The Posterior Predictive Sample Complexity with Probability $p$)
    Given the same setup in Theorem~\ref{th:finite_id_info_prob},
    the sample complexity distribution with at least probability $p$ of the posterior predictive distribution
    is determined by its moments, which are surely finite.
    The expected sample complexity is
    \begin{equation}
    \label{eq:id_with_prob_p_pred}
        \underset{\vvec{\psi}_i \sim \vvec{\Theta}_{i-1}}{
            \underset{\vvec{x}_i \sim X_i
            }{E}
        }[
        \mathbf{i}(\mathbf{X}^t_1, \vvec{x}^{t}_1, p, 0, 0, 0)
        ]
        =
        \frac{1}{H(\hat{X})} \Big( H(\vvec{\Theta}_0) - H( \vvec{\Theta}^t_1, \hat{X}^{t}_{1}) - \log_2(p) \Big)
    \end{equation}
\end{corollary}
\begin{proof}
    The ideal process is determined by $\psi$, whose expected description we now take with respect to the posterior process $\vvec{\psi}^t_1 \sim \vvec{\Theta}^t_1$.
    Starting from Equation~\ref{eq:id_with_prob_brevity},
    $$E_{\psi \in \vvec{\Theta}_0;~ \vvec{x}^t_1 \sim X^t_1 }[I(\vvec{\Theta}_0 = \psi) - \KL(X^t_1 || \hat{X}^t_1) ]$$
    the expected surprisal $E_{\psi \in \vvec{\Theta}_0}[I(\vvec{\Theta}_0 \mkern-3mu = \mkern-3mu \psi)] = tH(\vvec{\Theta}_0)$.
    Focusing on the remainder of the expectation, we express the \\ $E_{\vvec{\psi}^t_1 \in \vvec{\Theta}^t_1}[- \KL(X^t_1 || \hat{X}^t_1)]$ as in Equation~\ref{eq:id_with_prob_simplified}, which becomes
    \begin{align*}
        &=
            \sum^{t}_{i=1}
            \sum_{\vvec{\psi}_i \in \vvec{\Theta}_i}
            \mkern-5mu
            P(\vvec{\Theta}_{i-1} \mkern-5mu = \mkern-3mu \vvec{\psi}_i)
            \mkern-5mu
            \sum_{\vvec{x}_i \in \MeasurableInput}
            \mkern-5mu
            P(\hat{X}_i \!=\! \vvec{x}_i | \vvec{\Theta}_{i-1} = \vvec{\psi}_i)  
            \log
            \Bigg(
            \frac
            {\underset{{\theta \in \mathbf{\Theta}}}{\sum}
                P(\hat{X}_i \!=\! \vvec{x}_i | \vvec{\Theta}_{i-1}\!=\!\theta) P(\vvec{\Theta}_{i-1}\!=\! \theta)
            }
            {
                P(\hat{X}_{i} = \vvec{x}_i | \vvec{\Theta}_{i-1} = \vvec{\psi}_i)
            }
            \Bigg)
    \end{align*}
    We will separate the log fraction and simplify them to their respective entropies.
    Starting with the denominator, we get the following joint entropy
    $H(\vvec{\Theta}^{t-1}_{1}, \hat{X}^t_1)$
    \begin{equation*}
        \sum^{t}_{i=1}
        \sum_{\vvec{\psi}_i \in \vvec{\Theta}_i}
        \sum_{\vvec{x}_i \in \MeasurableInput}
        P(\hat{X}_i \!=\! \vvec{x}_i | \vvec{\Theta}_{i-1} = \vvec{\psi}_i)  
        P(\vvec{\Theta}_{i-1} \mkern-5mu = \mkern-3mu \vvec{\psi}_i)
        \log
        \big(
            P(\hat{X}_{i} = \vvec{x}_i | \vvec{\Theta}_{i-1} = \vvec{\psi}_i)
        \big)
    \end{equation*}   
    For the numerator, we get the entropy of the  posterior predictive distribution
    $H(\hat{X}^t_1)$
    \begin{align*}
        &=
        \sum^{t}_{i=1}
        \sum_{\vvec{\psi}_i \in \vvec{\Theta}_i}
        \mkern-5mu
        \sum_{\vvec{x}_i \in \MeasurableInput}
        \mkern-5mu
        P(\vvec{\Theta}_{i-1} \mkern-5mu = \mkern-3mu \vvec{\psi}_i)
        P(\hat{X}_i \!=\! \vvec{x}_i | \vvec{\Theta}_{i-1} = \vvec{\psi}_i)  
        \log
        \bigg(
        \underset{{\theta \in \mathbf{\Theta}}}{\sum}
            P(\hat{X}_i \!=\! \vvec{x}_i | \vvec{\Theta}_{i-1}\!=\!\theta) P(\vvec{\Theta}_{i-1}\!=\! \theta)
        \bigg)
    \end{align*}   
    Thus resulting in the conditional entropy subtracted from the prior distribution's entropy.
    \begin{equation*}
        H(\vvec{\Theta}_0) - H( \vvec{\Theta}^{t}_{1} | \hat{X}^t_{1})
    \end{equation*}
    Setting this expectation to $-\log_2(p)$ gives us Equation~\ref{eq:id_with_prob_p_pred}.
    As before, this expectation defines the $m$-th raw moment, which results in the sample complexity's distribution to be determined by its moments, as they are surely finite and the sample complexity is nonnegative.
\end{proof}
\begin{remark}
This is the sample complexity distribution that the predictor may use to inform itself of when the probable convergence has occurred, under the assumption that $\psi \in \mathbf{\Theta}$.
\end{remark}

With the moment generating function of the sample complexity defined by the summation of raw moments for identifying
    an i.i.d. random variable,
    the probability distribution of the sample complexity, as a non-negative process, can be defined using the inverse Laplace transformation as detailed in the notes by \citet{curtissNoteTheoryMoment1942} and \citet{gengTheoryMomentGenerating2020}.
However, given this would rely on computing an infinite series using the $m$-th raw moment, we have to compute enough such that the estimate is correct within our accepted probability.
Given the truncation would under estimate the probable sample complexity, it would be an optimistic estimate.
It is possible to compare the gain in precision by computing more, and once the gain is under the accepted error, then that is when the estimate is complete.

A notable difference to the typical statistical learning and PAC-Bayes framework to ours so far is that we have obtained the conditions for a probable finite sample complexity without the use of a bound directly on the error.
This is due to the use of a computable set $\mathbf{\Theta}$ that avoids the need to define a volume over the set because it is countable and each of its discrete elements have their own probability, while elements in $\R^d$ do not.
The volume is then the sum of the elements' probabilities.
Notably, it is known that if the hypothesis set $\mathbf{\Theta}$ was finite then so would be the sample complexity with $p, \epsilon \in (0, 1)$ with an upper bound of $\lceil \log(|\mathbf{\Theta}| / (1-p)) / \epsilon\rceil$, along the set having a finite VC dimension~\cite[Ch.~6.4]{shalev-shwartzUnderstandingMachineLearning2014}.
However, this upper bound is only informed by the number of hypotheses, while we've shown in the PAC-Bayes case that the complete sample complexity distribution is able to be determined by its moments and is informed by not only the number of known hypotheses, but also the Bayesian prior and their probability distributions of the observations. 
Depending on the considered hypothesis space, there may be multiple models all with posterior probability with at least $p$.

Theorem~\ref{th:finite_id_info_prob} and Corollary~\ref{th:pred_expect_with_prob} addressed a halting condition based on the Bayesian probability of a single model given the observations.
This can be extended to accept a subset of models and sum their probabilities together to find the probability of that subset as a mixture distribution.
This is where the information dissimilarity measure $D$ between models and the accepted error $\epsilon_D$ comes into play, where we consider the relative entropy $\KL$.
As mentioned discussing Theorem~\ref{th:asymp_sample_complexity}, the relative entropy between models determines how many observations one needs to tell them apart.
The more similar the models, the smaller their $\KL$, the closer their Bayes factors are to $1$, and the more observations necessary to differentiate them.
The Bayesian posterior updates given the evidence for whichever model is more likely amongst those considered to have generated the observations.
If the sample complexity is too large in number of observations and a degree of error is accepted, then by taking the subset of models that differ from one another within that accepted error is possible and can result in convergence in less observations, as less are needed to tell those similar models apart.
The procedure is then the same as with probability alone, except in this case, you group the similar models together and check their total posterior probability.

\subsection{Extension from i.i.d. to stationary ergodic processes}
So far, we have relied upon
$\lambda_S(\theta, t, \vvec{z}) = \vvec{x}^t_1$
to compute $t$ samples from an arbitrary i.i.d. random variable.
To extend our scope to a stationary ergodic process with finite $L$ memory we introduce the computation of its transition function $T$ by a Turing Machine with the description of $T$ denoted as
$\lambda_T(\delta, \vvec{x}^{-L}_t) = \theta_t$, where $\delta : \MeasurableInput^L \mapsto \mathbf{\Theta}$ is finite mapping of at most $2^L$ length strings to descriptions of i.i.d. variables.
$\lambda_S(\theta_t, 1, \vvec{z}) = x_t$ then computes the sample at each time-step.
We use $\lambda_{TS}(\delta, \vvec{x}^{-L}_t, i, \vvec{z}) = \vvec{x}^{i}_t$ to denote the chained alternating computations of $\lambda_T$ and $\lambda_S$ to compute an $i$ length sampling from a process whose description $\theta_t$ changes over the sampling conditioned on the moving window of $L$ prior samples.

To initiate computation of the process and to account for the $L$ prior state dependency, either $x^{-L}_1$ is provided along with $\delta$, or $\delta$ includes a means of determining $x^{-L}_1$.
In the case of a randomized initial state, $\delta$ may contain another distribution for sampling which of the $2^L$ possible random variables are sampled from for the initial $L$ symbols .
If $\delta$ describes the joint probability distribution of $P(\vvec{X}^{L+1}=\vvec{x}^{L+1})$, then
given any $L+1$ length string with missing symbols, the probable distributions for the missing samples can then be informed by $\delta$ and even sampled to fill in the missing samples, although we primarily consider the conditional $P(\vvec{X}_t=\vvec{x}_t|\vvec{X}^{-L}_t=\vvec{x}^{-L}_t)$ for computing the next output.
After the output is populated with $L$ symbols, the process simply reads the $L$ prior output to determine the next $\theta$ to sample from, and repeats.
This computation allows for any dependency on the prior $L$ symbols and also includes any deterministic process with $L$ memory.
With this computation in mind, we now show how such processes can be identified by predictors that can learn processes with less than or equal to $L$ memory.

\begin{corollary}
    The joint distribution
    $P(\vvec{X}^{L+1}_t)$, for any $t \in \Z$,
    identifies a stationary process with at most $L$ memory and is finitely described.
\end{corollary}
\begin{proof}
    Let $\vvec{X}$ be a stationary process with a finite memory of $L$ prior symbols $\vvec{x}^{-L}_t$.
    The joint probability distribution for $L+1$ length sequences at an arbitrary time step $t$ given the chain rule is
    \begin{align*}
        P(\vvec{X}^{L+1}_t=\vvec{x}^{L+1}_t) =
        P(\vvec{X}_{L+1} | \vvec{X}^{L}_t = \vvec{x}^{L}_t)
        P(\vvec{X}_{L} | \vvec{X}^{L-1}_t = \vvec{x}^{L-1}_t)
        \cdots \\
        \cdots P(\vvec{X}_{t+2} | \vvec{X}^{t+1}_t = \vvec{x}^{t+1}_t)
        P(\vvec{X}_{t+1} | \vvec{X}_t = \vvec{x}_t)
        P(\vvec{X}_{t} = \vvec{x}_t)
    \end{align*}
    As a stationary process, $P(\vvec{X}^{j}_t) = P(\vvec{X}^{j}_{t+\tau})$ for any time-shift $\tau \in \Z$ and for any sequence of length $j \in \Z^+$.
    Given the maximum dependency on $L$ memory, 
    for all $j \le L$, the joint $P(\vvec{X}^{L+1}_t)$ contains $P(\vvec{X}^{j}_{t+\tau})$ possibly with a rotational shift, capturing their probability.
    Thus, a stationary process with $L$ memory can be finitely described in a $|\MeasurableInput|^{L+1} \times |\MeasurableInput|^{L+1}$ matrix and the sequences $|\MeasurableInput|^{L+1}$ binary fractions of probability,
        which may serve as $\delta$ to describe the transition function in $\lambda_T(\delta, \vvec{x}^{-L}_t) = \theta_t$.
\end{proof}

Given the $P(X^{L+1})$ identifies an $L$ memory stationary stochastic process, if you assume ergodicity then you can take a set of possible such processes and from enough observations partially identify them with probability $p$ as done in the i.i.d. case in Theorem~\ref{th:finite_id_info_prob} where here the joint distribution determines the entropy rate.
Ergodicity ensures that by observing one sequence, you can eventually learn the finite description.
For non-ergodic stationary processes, a $|\MeasurableInput|^{L+1} \times |\MeasurableInput|^{L+1}$ matrix still defines them, but you cannot learn the entire description from observing a single process alone due to the lack of conservation of latent state.
You can only guarantee to learn a non-ergodic stationary process with $L$ memory if you can repeatedly see different observation sequences of the process as if it were a repeatable independent experiment.
In this case, the sequences are then finite as they must have some end in practice to observe more, and then you recover the i.i.d. case as first proven, but where a single sample is an entire such sequence.
$L$ sequential possibly non-identical random variables, as mentioned in the Theorem~\ref{th:asymp_sample_complexity}, is a special case of the $L$ memory stationary stochastic process, and so can be identified similarly.
This addresses the finiteness of the sample complexity with probability $p$ for stationary ergodic processes.
However, if $\psi \notin \mathbf{\Theta}$, then probable exhaustive falsification is required to identify that the samples belong to an unknown model.

\subsection{Probable Exhaustive Falsification: Probable Novelty}
\label{sec:prob_exhaust_false}

The concept of a misspecified hypothesis set
raises the question of ``How can one detect that the evidence is more likely generated from an unknown model rather than any of the known models?''
This is not a new question.
This is the same question asked by ``goodness of fit'' tests of a \textit{single} known model given the evidence, rather than comparing two or more models against each another as per Bayes factors, the Bayes information criterion, odds ratios, or maximum likelihood.
The above theorems already consider the case of comparison of multiple known models.
In the end, this problem is relatively more underspecified than the comparison of known models because we only know our considered models, and so can only observe how probable the observations are given those models.
Following the \textit{principle of maximum likelihood}~\citep{stiglerDanielBernoulliLeonhard1997,stiglerEpicStoryMaximum2007}, one can subjectively decide how likely such an observed sequence need be given the  known model, or the posterior mixture of known models.
This subjective decision is the same as determining the desired certainty in a confidence or credible interval, and in our case will determine the typical sets for the sample complexity in Definition~\ref{def:sample_complexity}.

    The Bayesian posterior probability will converge within the known set to the most probable model, or the most probable set of equiprobable models, given the observations.
    This occurs even if the ideal model is not known. 
    Given this, we examine the sample complexity distribution when $\psi \notin \mathbf{\Theta}$ and for the case of stationary ergodic processes we can use the empirical process as a representative distribution of the typical set of the unknown ideal process.
    From the difference between the ergodic process and the known model along with some acceptable level of error in bits, we can determine if a single known model's probability of the observations is close enough to the typical set to be accepted as known.

\begin{theorem}
    (Sample Complexity of 
    Probable Exhaustive Falsification)
    Let the computable set of hypothesized latent states be $\mathbf{\Theta} : |\mathbf{\Theta}| < \infty, \forall(\theta \in \mathbf{\Theta}):(\ell(\theta) < \infty)$, 
    $\mathbf{\Delta} \subseteq \MeasurableInput^L \times \mathbf{\Theta} : \{\delta \in \Delta : \MeasurableInput^L \mapsto \mathbf{\Theta}\}$,
    $\vvec{z} \sim Z$ be the results of flipping a fair coin,
    $\lambda_{TS}(\mathbf{\Delta}, \MeasurableInput^{L}, *, Z) = \vvec{\mathbf{X}}^*_1$
        be a set of stationary ergodic processes,
    the fixed finite process description $\psi \notin \mathbf{\Delta}$,
    and $\lambda_{TS}(\psi, \vvec{x}^{-L}_1, t, \vvec{z}) = \vvec{x}^t_1$
    be the observations of a
    stationary ergodic process.
    If $q > 0$,
    then
    $\mathbf{i}(\vvec{\mathbf{X}}^*_1, \vvec{x}^t_1, 1, q, 0, 0)$
        forms a distribution over $\Z^+$ that is determined by its moments, which are surely finite.
\end{theorem}
\begin{proof}
    When $p=1$, the condition $\exists(\theta \in \mathbf{\Theta}) : P(\Theta_t=\theta|X^t_1) \ge p$ will never occur in finite observations, as this returns to the asymptotic case of almost sure convergence.
    Thus, given $\psi \notin \mathbf{\Theta}$, the sample complexity $\mathbf{i}(\vvec{\mathbf{X}}^{*}_1, \vvec{x}^t_1, 1, q, 0, 0)$
    is fully determined by the exhaustive falsification condition $\forall(\theta \in \mathbf{\Theta}) :
        2^{-t(\entropyRate(\vvec{X}) + \log_2(q))} \le
        P(\vvec{X}^*_1 = \vvec{x}^t_1 | \Theta_t=\theta)
        \le 2^{-t(\entropyRate(\vvec{X}) - \log_2(q))}
        $
    .
    The expected surprisal of a single known model is the cross entropy when $m=1$
    \begin{align*}
        \underset{\vvec{x}^t_1 \sim X^t_1}{E}
        \Big[
        -\log_2 \big(
            P(\hat{X}^t_1 = \vvec{x}^t_1)
        \big)
        \Big] &= - \sum_{\vvec{x}^t_1 \in \MeasurableInput^t} 
            P(X^t_1 = \vvec{x}^t_1)
            \log_2 \big(
                P(\hat{X}^t_1 = \vvec{x}^t_1)
            \big)
        \\
        & = H_{\otimes}(X^t_1 || \hat{X}^t_1)
    \end{align*}
    and $m$-th raw moments are defined and finite as
    \begin{align*}
        \underset{\vvec{x}^t_1 \sim X^t_1}{E}
        \Big[
        \big(
        -\log_2 \big(
            P(X^t_1 = \vvec{x}^t_1)
        \big)
        \big)^m
        \Big] &=
            (-1)^m \sum_{\vvec{x}^t_1 \in \MeasurableInput^t} 
            P(X^t_1 = \vvec{x}^t_1)
            \Big(
            \log_2 \big(
                P(\hat{X}^t_1 = \vvec{x}^t_1)
            \big)
            \Big)^m
    \end{align*}
    By the typical set,
    $
    t(\entropyRate(\vvec{X}^t_1) - \epsilon_q)
    \le
    \underset{\vvec{x}^t_1 \sim X^t_1}{E}
        \Big[
        -\log_2 \big(
            P(\hat{X}^t_1 = \vvec{x}^t_1)
        \big)
        \Big]
    \le t(\entropyRate(\vvec{X}^t_1) + \epsilon_q)$. 
    The expected sample complexity is then
    \begin{align*}
        \frac{H_{\otimes}(\vvec{X}^t_1 || \hat{X}^t_1)}{\entropyRate(\vvec{X}^t_1) + \epsilon_q}
        &\le  t
         \le \frac{H_{\otimes}(\vvec{X}^t_1 || \hat{X}^t_1)}{\entropyRate(\vvec{X}^t_1) - \epsilon_q}
    \end{align*}
    With the $m$-th raw moments being defined and finite, the sample complexity is determined by its moments. 
\end{proof}

This addresses the evaluator's perspective where the evaluator knows the ideal distribution and the predictor's prior distribution, and thus can compute the sample complexity.
However, $\psi$ is unknown to the predictor, and thus the predictor must rely upon only what it knows to estimate the sample complexity.
    to assess if the  growing observation sequence is a member of any of the known models' typical sets~\cite[Ch.~3]{coverElementsInformationTheory1991}.
If the observations do not belong to any known model's typical set with the given accepted error, then the predictor determines the observations as belonging to an unknown model by exhaustive falsification with probability $q$.

    The typical set's complement defines the \textit{atypical set}, which is the union of a pair of disjoint sets, i.e., the \textit{atypically improbable set}  and \textit{atypically probable set} of observations for the given model.
    The earlier Figure~\ref{fig:typical_sets} visualizes an example's typical and atypical sets over the hypothesis' probability measure of the observations and shows how we allow for an undetermined set between the typical and atypical set of observations when $p > q$.
    To determine observations to be atypically improbable is to indicate that on average over the possible samples there exists a better model of those observations for which they are more probable.
    This informs to change the average direction of probable and improbable events.
    For atypically probable, there is a better model where the observations that are probable for this model are more probable and, similarly by conservation of probability, those that are improbable are more improbable.
    This indicates on average that the probability distribution has the correct ordering of the probability of events, but needs to increase the contrast between the improbable and probable events.
    When considering multiple known models' atypical sets and weighting them by a probability distribution, this informs how to update relative weighting and recovers the Bayesian estimation approach.
    When all known models indicate an atypical set, they serve as points in the probability space to inform the direction of the unknown ideal model relative to them, which is more informative than only knowing that the ideal is probably none of the known models.
    
    To check the set membership of the observations to the typical set of a known distribution is equivalent to
    comparing the difference between the empirical process $\Xi_t$ of $\vvec{x}^t_1$
    to the known model,
    where 
    $\KL\big(
        \Xi_t
        ||
        \hat{X}^t_1
    \big) > -\log_2(q)
    $ if the observations are atypical.
    The empirical process will, by definition, have the maximum probability for the observed sequence and as such can be relied upon to indicate the probable typical set of the ideal process.
    The empirical process reliably approximates the unknown model only when it has converged within the accepted degree of error ($-\log_2(q)$).
    Given the probability $q \in [0, p]$,
        the minimum
        precision of a single sample's probability before being considered observationally equivalent is 
        $-\log_2(1-q)$ bits for storing the largest possible number of occurrences of each unique symbol.
    That precision enables distinguishing the empirical process from distributions that differ with at least
        $|P(\Xi_t) - P(\hat{X})| > q$.
    This results in the maximum difference that a single element in the probability vector can differ to be
        $\frac{1}{2}q$
        as the rest of the probabilities would also have a maximum difference of
            $\frac{1}{2}q$, totaling $q$.
    The lowest common denominator of the probabilities $P(\hat{X})$ is the minimum sample complexity for zero error where the observed sequence empirically represents the probability distribution perfectly.
    Zero error can only occur periodically where the period is that lowest common denominator.
    If the lowest common denominator of observations transformed into bits is greater than the minimum precision determined by $-\log_2(1-q)$, then the desired precision is not fine-grained enough to adequately differentiate $\Xi_t$ from $\hat{X}$ and it is more likely for them to be partially verified as equivalent.
    Afterwards, each following observation will result in a change in the empirical process with a diminishing absolute difference in the probability that is less than $\frac{1}{-\log_2(1-q)}$, as every change in probability decreases with the growing number of observations given how the empirical process updates itself.
    
    Given the empirical process will not be able to measure within our accepted probability at least until it has observed $\entropyRate(\vvec{X}) - q$ bits, there is a ``warm up'' period necessary for more reliable evidence that the hypotheses are probably exhaustively falsified.
    Figure~\ref{fig:typical_sets:neq} depicts where the empirical process will estimate probability one for the single observation, which would put it well within one of the atypical sets for rejecting the hypothesis.
    Figure~\ref{fig:typical_sets:neq} is to visualize the sets clearly and demonstrate this important point that it does not make sense to compare the hypothesis to an under sampled representation as simpler hypotheses will favor smaller observation sequences due to fewer degrees of freedom required to be explained by evidence, i.e., the \textit{principle of parsimony}, popularly known as Occam's razor~\citep{charlesworthAristotleRazor1955}.

Practically, the membership of the observations to the known models' typical sets can be checked after convergence of the Bayesian posterior based on the stopping condition $p$ to determine if the converged known model assigns the observations a probability to be expected of typical observations for that partially identified model.
Given $q \le p$, the condition for exhaustive probability will be met, however if you consider the convergence of a subset of known models based on the accepted known error $\epsilon_D$, that is all models that form a subset able to be accepted differ from one another pairwise at most by $\epsilon_D$, then those subsets may have their probability summed to partially identify them faster than an individual model.
If an individual model is converged upon such that probability $p$ is satisfied, then checking the typical set membership of the observations to that known model will determine probable verification or falsification.
If comparing to a subset of known models, then check each of their typical sets and if all exclude the observation sequence, then exhaustively falsified, otherwise that subset is partially identified and further observations may be taken to further differentiate between those known models.
To help demonstrate how to use the sample complexity empirically, a simulated example will be available after publication at \url{https://github.com/prijatelj/sample_complexity_example}.
\section{Conclusion}
In this work, we explored the computation of hypothesis identification. 
Identifying novel information, as well as verifying hypotheses or their parts, whether in an absolute binary fashion or in degrees, enables learning the structure of the underlying data-generating process.
Fundamentally, identification is required for a machine to learn.
Identification also informs a predictor when its hypothesis set is misspecified, and thus in need of an update.
We unified the identification problem across various fields, including the theory of computation, asymptotic statistics, and Bayesian probably approximately correct learning from statistical learning theory.
We covered different cases of computing the indicator function within finite observations, from deterministic scenarios to ergodic stationary stochastic processes.
We demonstrated the properties of the sample complexity in these scenarios and proved that in the PAC-Bayes case, the sample complexity is determined by its moments, all of which are finite. 
This means that it is asymptotically bounded by an exponential decay, and may be computed to any desired accuracy based on the chosen prior.
In each of these cases, we addressed what novelty is and showed that it is identified by (probable) exhaustive falsification of the known models given the observations.

From this work, we have more complete answers to the questions raised in the abstract.
Information is determined by the number of distinct states and their relationships that can be expressed in a language, such as a programming language or a formal language. 
The information provided by the observations is relative to the observation space and the set of hypotheses considered, whose superset is determined by the chosen formal language, programming language, and resource limitations, such as memory space limitations.
The hypothesis set and its representation within your language determine how many observations remain to verify that an element in the set corresponds to the observations or exhaustively falsify the hypothesis set, which is to identify novelty.
Although our formalization of identifying information with probability and error for a fixed finite hypothesis set
covers the \textit{identification} of novelty, this does not address the \textit{learning} of a new hypothesis from observations, as in the case of a finitely growing or changing hypothesis set.
We have established the necessary foundation for such a work to be built upon, which will be explored in the future.

\section*{Acknowledgements}
    Thank you to those who have reviewed drafts of this paper in parts or in full and those who provided valuable feedback to improve the resulting document.

\appendix

\section{Background and Related Works}
\label{sec:app:bg}

Given this paper's concepts are of historic and recent interest across many disciplines,
we amusingly encounter a ``curse of interdisciplinary study,'' which is an instance of the ``curse of dimensionality''~\cite[Ch. 5.16]{bellmanAdaptiveControlProcesses1961}.
In this case, the more disciplines in which researchers are conducting relevant work, the more that must be read to understand their interrelations and contributions to the whole.
If one field provides unique information absent in another, then a unified perspective is beneficial.
As such, having context across these disciplines provides a more complete picture of the existing theories of learning pertaining to an information theoretic perspective.
We summarize them to establish the foundation of concepts that we build our contributions upon.


\subsection{Information Theory}

Information is our primary concern and measures of what information is contained where or shared between what processes is crucial to assess learning.
\citet{kolmogorovThreeApproachesQuantitative1968} showed that information may be quantified through three approaches: combinatorial, probabilistic, and algorithmic.
    Measurements of information record the amount of space, as in the number of symbols, required to differentiate the distinct states of some phenomenon from one another.
When represented as binary strings, the symbols are encoded in bits.
The seminal work by \citet{shannonMathematicalTheoryCommunication1948} formalized information theory and its measures in terms of communication channels and the probability distribution over the symbols communicated over those channels.
Shannon defined information through the relationship of distinct states to one another and their optimal encodings based on their probability in terms of frequency of their occurrence.
The information theoretic measures of interest are in Table~\ref{tab:info_theory_bg}, all of which are covered by~\citet[Ch.~2]{coverElementsInformationTheory1991} and \citet[Ch.~2]{mackayInformationTheoryInference2003}, while block entropy, which is the entropy of a sequence, is well covered by~\citet{crutchfieldRegularitiesUnseenRandomness2003}.
These concepts are also explored in application in quantitative linguistics~\cite{FrontMatter2020}.
Of special interest is the \citet{kullbackInformationSufficiency1951} divergence $\KL(X_i || X_j)$, or the \textit{relative entropy}, which is an asymmetric measure of the entropy of $X_i$ relative to $X_j$, both with support over $\MeasurableInput$.
The relative entropy can be interpreted as the average amount of extra bits necessary to communicate the information in $X_i$ in an optimal encoding for $X_j$.

\begin{table}[t]
    \centering
    \caption{\textsc{Information theoretic measures.}}
    \fontsize{10}{18}\selectfont
    \begin{tabular}{r|ll}
        Surprisal &
        $
        I(\Input=\sampleInput) $ & $\triangleq -\log_2(P(\Input=x))
        $
        \\
        Entropy &
            $
                H(X) $ & $\triangleq - \sum_{x \in \MeasurableInput} P(X = x) \log_2 P(X=x)
            $
        \\
        Block Entropy &
        $
        H(\vvec{X}^t_1) $ & $\triangleq - \sum_{\vvec{x}^t_1 \in \MeasurableInput^t} P(\vvec{X}^t_1 = \vvec{x}^t_1) \log_2 P(\vvec{X}^t_1 = \vvec{x}^t_1)
        $
        \\
        Entropy Rate &
        $
        \entropyRate(\vvec{X}_t) $ & $\triangleq \lim_{L \rightarrow \infty}\frac{H(\vvec{X}^{L}_t)}{L}
        $
        \\
        Cross Entropy &
        $
        H_\otimes(X_i || X_j)$ & $\triangleq -\sum_{x \in \MeasurableInput} P(X_i = x) \log_2 P(X_j = x)
        $
        \\
        Relative Entropy &
        $
        \KL(X_i || X_j) $ & $\triangleq \sum_{x \in \MeasurableInput} P(X_i = x) \log_2 \frac{P(X_i=x)}{P(X_j=x)}
        $
        \\
        Conditional Entropy &
        $
        H(X | Y) $ & $\triangleq - \sum_{x \in \MeasurableInput}\sum_{y \in \MeasurableOutput} P(X = x, Y=y) \log_2 P(X=x | Y=y)
        $
        \\
        Joint Entropy &
        $
        H(X, Y) $ & $\triangleq  - \sum_{x \in \MeasurableInput}\sum_{y \in \MeasurableOutput} P(X = x, Y=y) \log_2 P(X = x, Y=y)
        $
        \\
        Mutual Information &
        $
        I(X ; Y) $ & $\triangleq \sum_{x \in \MeasurableInput} \sum_{y \in \MeasurableOutput} P(X = x, Y = y) \log_2 \frac{P(X=x, Y=y)}{P(X=x)P(Y=y)}
        $
        \\
    \end{tabular}
    \label{tab:info_theory_bg}
\end{table}

\subsection{Computability and Algorithmic Complexity}
Defining and reasoning about the machine learning process inherently is concerned with what can be computed, as apparent by its name.
We mentioned algorithmic complexity also defines measures of information in terms of the number of symbols, covered in-depth by~\citet[Eq.~14.1]{coverElementsInformationTheory1991} and \citet[Ch.~2]{liIntroductionKolmogorovComplexity2019}.
Algorithmic complexity, as a field of study, focuses on the sets of symbols and sets of strings, which are sequences composed of those symbols.
These symbols are all with respect to some machine that takes them as input and returns some output string, such as a universal Turing machine.
We purposefully use the term ``computable'' instead of ``recursive'' given our concern is with what can be represented by the calculations of a Turing Machine as covered in~\citet{soareComputabilityRecursion1996}, although in the beginning of Section~\ref{sec:id_info} we consider strings and sets that may be countably infinite in length or size for theoretical completeness.

The following are fundamental definitions and properties of what is computable or computably enumerable (c.e.) as discussed in various works, where the original influential papers are reprinted in \citet{martinUndecidableBasicPapers1965} and further covered by \citet{minskyComputationFiniteInfinite1967}, \citet{rogersTheoryRecursiveFunctions1987}, and \citet{soareTuringComputability2016}.
\begin{itemize}[nolistsep,noitemsep]
    \item A computable function is a function whose output can be computed by a Turing machine that then halts~\citep{turingComputableNumbersApplication1937}.
    \item A computable set is a set for which the indicator function can be computed for any query string and halt.
        A computable set and its complement are both c.e. as shown by the \citet{postRecursivelyEnumerableSets1944}~Complementation Theorem~\cite[Theorem~2.1.14]{soareTuringComputability2016}.
    \item A c.e. set is a set for which the indicator function can compute verification and halt if the element is in the set otherwise the computation may not falsify and halt, for example due to an infinite set of items or a circular loop with no stopping condition met.
        A complement c.e. set is the opposite, where it can always falsify, but not always verify, set membership of a query string.
        A c.e. set is the range of a total computable function.
    \item A computable number is a string whose finite $n$ length prefix string can computed by a Turing Machine and halt~\citep{turingComputableNumbersApplication1937}.
\end{itemize}
Note that there exists noncomputable c.e. sets~\cite[Ch.~1.6.2]{soareTuringComputability2016}, however we focus on the computable case.
We care about what a machine can actually observe, describe, and output.
Computability is defined with respect to the indicator function, as we do for identification, and we further address this connection throughout the paper.

The space of the computable numbers $\K$ is the union of the computable irrationals with the rationals $\Q \triangleq \frac{\Z}{\Z^+}$, where $\Z$ is the set of integers,
    $\Z^+$ is the set of positive integers $\{1, 2, 3,...\}$.
$\N$ is the natural numbers starting at $0$, $\N \triangleq \{0\} \cup \Z^+$. 
$\K$ is a subset of the real numbers $\R$.
The proper subset relationship chain from $\N$ to complex numbers $\C$ is
$
    \N \subset \Z \subset \Q \subset \K \subset \R \subset \C
$.
As the vast majority of numbers in $\R$ are not computable, a $d$-dimensional space of real numbers $\R^d$ is also not entirely enumerable.
All sets considered in this theory are subsets of $\K^d$, where $d \in \Z^+$.
This includes the probability measures.
The nuance in the loss of information due to the computable representation when the ideal system may actually involve $\R$ is the concern of rate-distortion theory~\cite[Ch.~10]{coverElementsInformationTheory1991}, \cite{bergerRateDistortionTheory1971}
and in such cases only computable approximations are possible as further detailed by \citet[Ch.~2.2]{hutterUniversalArtificialIntelligence2004}.
We often concern ourselves with binary strings where a bit is within the Boolean domain $\B \triangleq \{0,1\}$.
The set of $L$ length bit strings are denoted $\B^L$ and the set of arbitrary length bit strings are denoted with a Kleene star $\B^* \triangleq 2^{\B}$.

\subsection{Statistics, Information, and Identification}
Statistics can be defined in terms of probability theory, as it often is, or, as done through the Kolmogorov structure function~\citep{vereshchaginKolmogorovStructureFunctions2004}, in non-probabilistic terms with sets and algorithms~\cite[Ch.~4]{liIntroductionKolmogorovComplexity2019}, \cite[Ch.~17]{vovkMeasuresComplexityFestschrift2015}. 
This yields \textit{probabilistic statistics} and \textit{algorithmic statistics}.
Statistics are functions
    of the observations to state spaces which characterize the information of properties of their respective unknown ideal process.
The subfield of distribution testing covered by \citet{goldreichIntroductionPropertyTesting2017} and surveyed by \citet{canonneSurveyDistributionTesting2020,canonneTopicsTechniquesDistribution2022} emphasizes this point: the properties of a distribution can be tested from limited observations separately from the whole distribution.
Property testing has been shown to be connected to statistical learning~\citep{goldreichPropertyTestingIts1998}.
As we noted earlier, information measures defined in terms of either probability or algorithmic complexity both measure the space required, in number of symbols, to represent distinct states of the phenomenon.
They both yield useful perspectives on information, where probability relates the events in terms of the certainty of their occurrence with respect to other mutually exclusive events.
Algorithmic complexity relates the events to their computable representation with respect to some machine that generates them from their description, and their string lengths become apparent for that machine and its programming language.
Both approaches result in the same minimum description length encoding when encoding events based on how frequent or likely it is to see an event with respect to other events.
This only holds true when the probability distribution of the events matches the actual frequency of events, thus 
    a frequentist perspective, an objective Bayesian perspective, or a subjective Bayesian perspective are all applicable when enough observations occur,
    and all of these perspectives yield the same asymptotic random variable estimate in the infinite observations assuming the models are correct and can perfectly match the ideal~\cite[Ch.~6]{ghosalFundamentalsNonparametricBayesian2017}. 

Statistical model identification is a well studied area, even in terms of information theory, as explored by \citet{akaikeNewLookStatistical1974,akaikeInformationTheoryExtension1998}.
However, identifiability has not been studied to the degree of detail for which we specify identifiable information and how it is identified from observations.
\citet[Sec.~2]{lewbelIdentificationZooMeanings2019} briefly covers the history of identifiability through 1662 into the 1970s from economics to mathematical systems modeling.
Our work may in part be seen as a refinement via computation and information theory of the learnability or identification in the limit of languages by \citet{goldLimitingRecursion1965,goldModelsGoalSeekingLearning1965,goldLanguageIdentificationLimit1967}, which has developed into algorithmic learning theory as covered by~\citet{oshersonSystemsThatLearn1985}.
We prioritize examination of finite sample complexity, rather than in the limit, and address when limiting behavior may occur.

The textbooks by \citet{walterIdentificationParametricModels1997} and \citet{isermannIdentificationDynamicSystems2011} cover a variety of model types including linear and nonlinear, discrete and continuous, and observations with and without noise, all within the context of dynamical system modeling and control theory.
\citet[Ch.~2.6]{walterIdentificationParametricModels1997} puts forth identifiability along with distinguishability, which is simply another type of identification, where in their case distinguishability is about whether one parameterized model is observationally equivalent to another different parameterized model for some subset of their parameters.
If they are equivalent, then they are indistinguishable, otherwise they are distinguishable.
Those in statistical learning may consider these as at least partially overlapping hypothesis spaces or classes.
\citet[Ch.~1]{isermannIdentificationDynamicSystems2011} provide a detailed account on systems identification clearly explaining the relationships between theoretical and experimental modeling and how identification comes into play.
Information theoretic measures are known to capture both linear and nonlinear correlations between variables.
Given this, our work refines the concept of identification presented by these prior works such that we can reason about how to identify different parts or properties through their information in the observations.

\subsection{Statistical and Computational Learning}
Identification and sample complexity play a key role in statistical and computational learning theories.
The \citet{vapnikUniformConvergenceRelative2015} (VC) theory is often considered the start of statistical learning theory as a field~\cite[Ch.~2]{vovkMeasuresComplexityFestschrift2015}, \cite[Ch.~3.4]{shalev-shwartzUnderstandingMachineLearning2014}. 
Which has grown to include PAC learning.
Prior and independent to VC theory was the algorithmic language identification and learnability work of \citet{goldLimitingRecursion1965,goldModelsGoalSeekingLearning1965,goldLanguageIdentificationLimit1967} which
    initiated computational or algorithmic learning theory as a field, along with \citet{putnamPhilosophicalPapersVolume1975a} and \citet{solomonoffFormalTheoryInductive1964,solomonoffFormalTheoryInductive1964a} introducing algorithmic inductive inference~\citep{oshersonSystemsThatLearn1985}.
Identification and learnability are one and the same in the algorithmic learning and inductive inference literature.
A Bayesian variant of PAC learning introduced by \citet{mcallesterPACBayesianTheorems1998,mcallesterPACBayesianModelAveraging1999}  resulted in PAC-Bayes, which is covered by \citet[Ch.~31]{shalev-shwartzUnderstandingMachineLearning2014}, \citet{guedjPrimerPACBayesianLearning2019}, and \citet{alquierUserfriendlyIntroductionPACBayes2024}.
Our work recovers computable identification in the PAC-Bayes case for ergodic stationary processes and proves that the PAC-Bayes sample complexity distribution is determined by its moments.

Although PAC learning was originally considered under the constraints of computability in polynomial time by \citet{valiantTheoryLearnable1984}, most others explored PAC learning without those constraints.
Recently, computable PAC learning has been reconsidered.
\citet{ackermanComputabilityConditionalProbability2019}  address the computability of the conditional probability over both continuous and countable spaces.
We consider primarily countable spaces.
\citet{agarwalLearnabilityWihComputable2020} and \citet{sterkenburgCharacterizationsLearnabilityComputable2022} consider
    where the decidability and solvability of a learning problem by an algorithm is explored when the only constraint is that the hypotheses and the sample complexity are all computable.
There has also been studies of computable online learning~\citep{hasratiComputableOnlineLearning2023} and computable PAC learning of continuous features~\citep{ackermanComputablePACLearning2022}.
The exploration of computational complexity of PAC learning was continued by \citet{valiantProbablyApproximatelyCorrect2013} and was further studied by others, including those in property testing as previously mentioned~\cite{goldreichPropertyTestingIts1998}.
The exploration of our results with sample complexity with these works is promising for future work.
In our work we show how identifiability relates across learning deterministic and nondeterministic stationary processes in order to  better understand when identification can occur and how the probability and error from PAC comes into play to recover a finite sample complexity.
We are concerned with the general learning problem where given a computable predictor with finite space constraints, how does this effect the sample complexity?
We better connect the sample complexity to the known hypothesis set, as information and thus the sample complexity is fundamentally relative to the description language used by the predictor and what it knows.
\section*{Depth-first Identification on an Unsorted Set}
\begin{algorithm}[t]
\caption{Modified Indicator Function for Identifying Binary Strings: Unsorted Set}
\label{alg:id_depth_first_set}
    \begin{algorithmic}[1]
    \Procedure{Identify}{$\mathbf{\Theta}, \vvec{\theta}, r$} \Comment{indexed set, query string, \& resolution}
    \State $i \gets 1$ \Comment{The observed information to identify if in or out of set is at least one bit}
    \State $j \gets 0$ \Comment{The index or number of elements checked to identify}
    \State $h \gets 0$ \Comment{The index of string with max shared prefix}
    \For{$\vvec{\psi} \in \mathbf{\Theta}$} \label{alg:id:loop_set} \Comment{Loop through checking each parameter depth-first}
        \State $j \gets j + 1$
        \For{$k \gets 1; k \le \ell(\vvec{\theta}) \And k \le -\log_2(r); k \gets k + 1$} \Comment{Check each bit in order}
            \If{$k > i$}
                \State $i \gets k$ \Comment{Save the max shared prefix size as identifiable info ...}
                \State $h \gets j$ \Comment{... and its enumerated index}
            \EndIf
            \If{$\vvec{\psi}_k \neq \vvec{\theta}_k$} \Comment{Current string in set is falsified as a match to $\vvec{\theta}^{-\log_2(r)}_1$}
                \State \Goto{alg:id:loop_set}
            \EndIf
        \EndFor \Comment{Reach loop's end $\!\iff\!$ the strings are observationally equivalent}
        \State \textbf{return} $(1, j, i)$ \Comment{$\mkern-4mu$Shared prefix length min$(\ell(\vvec{\theta}),-\log_2(r)) < \infty$. Returns \&  halts}
    \EndFor \Comment{Reach outer loop's end $\implies |\mathbf{\Theta}| < \infty$}
    \State \textbf{return} $(0, h, i)$ \Comment{No-match indicator, best match index, \& shared prefix length}
    \EndProcedure
    \Comment{\!Halts $\!\iff\!$ finite subset $\{\vvec{\psi} \in \mathbf{\Theta}\}$: each sharing a finite prefix with $\vvec{\theta}$}
    \end{algorithmic}
\end{algorithm}

Algorithm~\ref{alg:id_sorted_set} assumes that $\mathbf{\Theta}$ is sorted.
Algorithm~\ref{alg:id_depth_first_set} does not have that assumption, which comes at the cost of not being able to rely upon the structure of the strings' ordering within the set for identification.
\section*{Pairwise Identify Countably Infinite Strings with Random Order}
\label{sec:app:id_inf_direct_obs}
\begin{theorem}
\label{th:id_pair_infinite_unordered}
    (The Sample Complexity Distribution
    to
    Pairwise Identify
    Infinite Strings)
    Let there be a pair of strings $\theta$ and $\psi$ of countably infinite length whose order of symbol comparisons is random.
    \begin{itemize}[noitemsep,nolistsep]
        \item If the amount of equivalent and non-equivalent pairwise comparisons between two strings are both countably infinite, then the sample complexity's probability distribution is a geometric distribution with $p$ as the probability of a non-equivalent symbol comparison:
        \begin{equation}
        \label{eq:id_inf_geo}
            P\Big(\mathbf{i}\big(\{\psi^L_1\}, \theta^L_1 \big) =i\Big) = (1-p)^{i-1}p
        \end{equation}
        \item If the equivalent symbols are finite while the non-equivalent symbols are infinite, then the strings are almost surely determined to be not equal.
        \item If the equivalent symbols are infinite while the non-equivalent symbols are finite, then almost surely the machine will not halt comparing the strings' symbols.
    \end{itemize}
    
\end{theorem}
\begin{proof}
    Let $\theta$ and $\psi$ be pairwise symbol aligned strings with a countably infinite length.
    As in the proof for Theorem~\ref{th:id_pair_finite_unordered}, each observations follows a Bernoulli where $p$ is the probability of a non-equivalent symbol pair.
    However, that probability remains constant as observations occur due to the strings' infinite length, which results in $\mathbf{i}\big(\{\psi\}, \theta \big) \in \Z^+$ and the probability distribution as in Equation~\ref{eq:id_inf_geo}.
    When $p \rightarrow 1$, then the symbol comparisons sampled are almost surely not equal and certainly not equal when $p=1$.
    When $p \rightarrow 0$, then the symbol comparisons are almost surely equal and certainly equal when $p=0$.
\end{proof}
\begin{remark}
While the sample complexity distribution could be infinite in the case of comparing infinitely long strings with unknown order by sampling their pairs as above, selecting any ordering and then performing an iterative comparison over that order will always result in finite observations if the Bernoulli's $p \ge 0$, which is knowledge not known without running the identification program or otherwise obtaining that information somehow, such as being informed by the order of the strings by their value.
Considering set sizes and string lengths of countably infinite size is for the theoretical completion of identification and sample complexity and provides insight into what it means to know, i.e., to determine, certain information. 
\end{remark}

Besides the order of symbol comparisons, the order in which the strings are to be compared may also be modeled similarly by the distributions in Theorems~\ref{th:id_pair_finite_unordered} and \ref{th:id_pair_infinite_unordered}.
If the set is countably infinite, then exhaustive falsification of the set is impossible without information about the strings prior to observing them in their entirety, such as learning the partial order of the query string to the rest of the set's elements given those falsified, and then only verification will result in a termination of comparisons, although subsets of the possible strings may be falsified.

\printbibliography

\end{document}